\documentclass[accepted]{uai2021} 

\usepackage[american]{babel}

\usepackage{natbib} 
\usepackage{booktabs} 
\usepackage{tikz} 

\usepackage[utf8]{inputenc} 
\usepackage[T1]{fontenc}    
\usepackage{hyperref}       
\hypersetup{
	colorlinks   = true, 
	urlcolor     = black, 
	linkcolor    = black, 
	citecolor   = black 
}

\usepackage{url}            
\usepackage{booktabs}       
\usepackage{amsfonts}       
\usepackage{nicefrac}       
\usepackage{microtype}      

\usepackage{algorithm, algorithmic}

\usepackage{amsfonts,amsmath,amsthm,amssymb}
\allowdisplaybreaks
\usepackage{bm}
\newtheorem{theorem}{Theorem}
\newtheorem{lemma}{Lemma}

\newtheorem{proposition}{Proposition}

\newtheorem{definition}{Definition}
\newtheorem{assumption}{Assumption}

\usepackage{subfigure}
\usepackage{booktabs}
\usepackage{makecell}
\usepackage{setspace}

\newcommand{\E}{\mathbb{E}}
\newcommand{\I}{\mathbb{I}}
\newcommand{\R}{\mathbb{R}}

\newcommand{\bX}{\boldsymbol{X}}

\newcommand{\bmu}{\boldsymbol{\mu}}
\newcommand{\cA}{\mathcal{A}}
\newcommand{\cB}{\mathcal{B}}

\newcommand{\cE}{\mathcal{E}}
\newcommand{\cF}{\mathcal{F}}

\newcommand{\cN}{\mathcal{N}}

\newcommand{\one}{\bm{1}}

\newcommand{\mb}[1]{\mathbb{#1}}
\newcommand{\mc}[1]{\mathcal{#1}}
\newcommand{\mr}[1]{\mathrm{#1}}

\usepackage{ifthen}
\usepackage{xcolor}

\usepackage{thm-restate}

\usepackage{tikz}
\usetikzlibrary{positioning,chains,fit,shapes,calc}

\newcommand{\argmax}{\operatornamewithlimits{argmax}}

\newcommand{\inner}[1]{ \left\langle {#1} \right\rangle }

\ifthenelse{\equal{compilefullversion}{false}}{%
    \newcommand{\OnlyInFull}[1]{}
    \newcommand{\OnlyInShort}[1]{#1}
}{%
    \newcommand{\OnlyInFull}[1]{#1}%
    \newcommand{\OnlyInShort}[1]{}%
}%

\newcommand{\compilehidecomments}{false}
\ifthenelse{ \equal{\compilehidecomments}{true} }{%
    \newcommand{\wei}[1]{}
    \newcommand{\haoyu}[1]{}
    \newcommand{\kai}[1]{}
}{
    \newcommand{\wei}[1]{{{O}lor{blue!50!black}  [\text{Wei:} #1]}}
    \newcommand{\haoyu}[1]{{{O}lor{brown!60!black} [\text{Haoyu:} #1]}}
    \newcommand{\kai}[1]{{{O}lor{brown!60!black} [\text{Kai:} #1]}}
}

\mathchardef\mhyphen="2D
\newcommand{\Oracle}{{\sf O}}
\newcommand{\Cucbsw}{{\sf CUCB\mhyphen SW}}
\newcommand{\Cucbbob}{{\sf CUCB\mhyphen BoB}}

\title{Combinatorial Semi-Bandit in the Non-Stationary Environment}

\author[1,*]{\href{mailto:Wei Chen <weic@microsoft.com>?Subject=Your UAI 2021 paper}{Wei Chen}}
\author[2]{\href{mailto:Liwei Wang <wanglw@cis.pku.edu.cn>?Subject=Your UAI 2021 paper}{Liwei Wang}}
\author[3]{\href{mailto:Haoyu Zhao <haoyu@princeton.edu>?Subject=Your UAI 2021 paper}{Haoyu Zhao}}
\author[4]{\href{mailto:Kai Zheng <zhengk92@gmail.com>?Subject=Your UAI 2021 paper}{Kai Zheng}}
\affil[1]{%
	Microsoft Research, Beijing, China. \texttt{weic@microsoft.com}
}
\affil[2]{%
	Key Laboratory of Machine Perception, MOE, School of EECS\\

	Center for Data Science, Peking University, Beijing, China. \texttt{wanglw@cis.pku.edu.cn}
}
\affil[3]{
	Princeton University, NJ, USA. \texttt{haoyu@princeton.edu}
}

\affil[4]{Kuaishou Inc., Beijing, China. \texttt{zhengk92@gmail.com}}
\affil[*]{Alphabetic order}

\begin{document}
\maketitle

\begin{abstract}
  In this paper, we investigate the non-stationary combinatorial semi-bandit problem, both in the switching case and in the dynamic case. In the general case where (a) the reward function is non-linear, (b) arms may be probabilistically triggered, and (c) only approximate offline oracle exists \citep{wang2017improving}, our algorithm achieves $\tilde{O}(m\sqrt{N T}/\Delta_{\min})$ distribution-dependent regret in the switching case, and $\tilde{O}({V}^{1/3}T^{2/3})$ distribution-independent regret in the dynamic case, where ${N}$ is the number of switchings and ${V}$ is the sum of the total ``distribution changes'', $m$ is the total number of arms, 
  and $\Delta_{\min}$ is a gap variable dependent on the distributions of arm outcomes. 
  The regret bounds in both scenarios are nearly optimal, but our algorithm needs to know the parameter ${N}$ or ${V}$ in advance. 
    We further show that by employing another technique, our algorithm no longer needs to know the parameters  ${N}$ or ${V}$ but the regret bounds could become suboptimal.
    In a special case where the reward function is linear and we have an exact oracle, we apply a new technique to design a parameter-free algorithm that achieves nearly optimal regret both in the switching case and in the dynamic case without knowing the parameters in advance.
\end{abstract}

\section{Introduction}\label{sec:intro}
Stochastic multi-armed bandit (MAB) \citep{auer2002finite,thompson1933likelihood} is a classical model that has been extensively studied in online learning and online decision making. The most simple version of MAB consists of $m$ arms, where each arm corresponds to an unknown distribution. 
In each round, the player selects an arm, and the environment generates a reward of that arm from the corresponding distribution. The objective is to sequentially select the arms in each round and maximize the total expected reward. The MAB problem characterizes the trade-off between exploration and exploitation: 
On the one hand, one may play an arm that has not been played much before to explore whether it is good, and on the other hand, one may play the arm with the largest average reward so far to accumulate the reward.

Stochastic combinatorial multi-armed bandit (CMAB) is a generalization of the original stochastic MAB problem. 
In CMAB, the player may choose a combinatorial action over the arms $[m]$, and thus there may be an exponential number of actions. 
Each action triggers a set of arms, the outcomes of which are observed by the player.
This is called the {\em semi-bandit} feedback.
Moreover, some arms may be triggered probabilistically based on the outcome of other arms \citep{CWYW16,wang2017improving,kveton2015cascading,kveton2015combinatorial}. 
CMAB has received much attention because of its wide applicability
    from the original online (repeated) combinatorial optimization to other practical problems, e.g. wireless networking, 
    online advertising, recommendation, and influence maximization in social networks \citep{CWY13,CWYW16,wang2017improving,gai2012combinatorial,combes2015combinatorial,kveton2014matroid,kveton2015cascading,kveton2015combinatorial,kveton2015tight}.

All these studies focus on the stationary case, where the distribution of arm outcomes stays the same through time. 
However in practice, the environment is often changing. 
For example, in network routing, some routes are not available temporarily for maintenance; in influence maximization, 
    student users may likely use social media less frequently during the final exam period; in online advertising and recommendation, people's preferences may change due to news events or fashion trend changes.

Motivated by such realistic settings, we consider the non-stationary CMAB problem in this paper. 
Let ${D}_t$ denote the distribution of the arm outcomes (represented as a vector) at time $t$.
We use two quantities, switchings and variation, to measure the changing of distributions $\{{D}_t\}_{t\le T}$. 
The number of switchings is defined as ${N} := 1 + \sum_{t=2}^{T}\I\{{D}_t \neq {D}_{t-1}\}$, and the variation is given as ${V} :=  \sum_{t=2}^{T}||\bmu_t -\bmu_{t-1}||_{\infty}$, 
    where $\bmu_t$ is the mean outcome vector of the arms following distribution ${D}_t$. 
A related definition is the total variation $\bar{V} := \sum_{t=2}^{T}||{D}_t -{D}_{t-1}||_{\text{TV}}$, where $||\cdot ||_{\text{TV}}$ denotes the total variation of a distribution.
The performance of the algorithm will be measured by the \emph{non-stationary regret} instead of the regret in the stationary case.

This problem is first considered by \cite{zhou2019near}, where the authors consider the non-stationary CMAB with approximation oracle but no probabilistically triggered arms. \citet{zhou2019near} only study the switching case, or the piecewise stationary case, where the non-stationarity is measured by ${N}$. Moreover, they add an assumption on the length of each stationary segment and thus bound the switchings ${N}$ to be ${O}(\sqrt{T})$. Different from their model and assumptions, we consider the non-stationary CMAB in both the switching case (measured by ${N}$) and the dynamic case (measured by ${V}$ or $\bar{V}$). We do not make assumptions on the number of switchings ${N}$ and the length of stationary periods. Our contributions can be summarized as follow:
 
    \textbf{1.} When we know the changing parameters ${N}$ or ${V}$, we design algorithm $\Cucbsw$ for the non-stationary CMAB problem. We show that $\Cucbsw$ has nearly optimal distribution-dependent bound both in the switching case and the dynamic case, and the leading terms in the regret bounds are $\tilde{O}(m\sqrt{{N} T}/\Delta_{\min})$ and $\tilde{O}(m\sqrt{{V} T}/\Delta_{\min})$, where $m$ is the total number of arms and $\Delta_{\min}$ is gap variable dependent on the distributions of arm outcomes
    (see Section~\ref{sec:general} for the precise technical definition). 
    We also show that $\Cucbsw$ has nearly optimal distribution-independent bound in the dynamic case and the leading term in the bound is $\tilde{O}({V}^{1/3}T^{2/3})$. 
    
    \textbf{2.} When parameters ${N}$ or ${V}$ are unknown, we design algorithm $\Cucbbob$, which achieves sublinear regret in terms of $T$ as long as ${N} < cT^{\gamma}$ or ${V} \le cT^{\gamma}$ for some constants $c$ and $\gamma < 1$. Moreover, the distribution-dependent bounds in both cases and the distribution-independent bound in the dynamic case are nearly optimal when ${N}$ and ${V}$ are large.
    
    \textbf{3.} In a special case when (a) the total reward of an action is linear in the means of arm distributions, 
    	(b) there is no probabilistically triggered arms, and 
    	(c) we have an exact oracle for the offline problem, 
    	we design \textsc{Ada-LCMAB} that does not need to know the parameters ${N}$ or ${V}$ in advance. Our algorithm has 
    	distribution-independent regret bounds $\tilde{O}(\min\{\sqrt{{N} T},{V}^{1/3}T^{2/3} + \sqrt{T}\})$, which is nearly optimal in terms of ${N}$, ${V}$, $T$ in both the switching case and the dynamic case.

\subsection{Related works}
\paragraph{Multi-armed bandit} Multi-armed bandit (MAB) problem is first introduced in
\cite{robbins1952bulletin}. MAB problems can be classified into stochastic bandits and adversarial bandits. 
In the stochastic case, the reward is drawn from an unknown distribution, and in the adversarial case, the reward is determined by an adversary. 
Our model is a generalization of the stochastic case, as discussed below.
The classical MAB algorithms include UCB \citep{auer2002finite} and Thompson sampling \citep{thompson1933likelihood} for the stochastic case and EXP3 \citep{auer2002nonstochastic} for the adversarial case. We refer to \cite{bubeck2012regret} for a comprehensive coverage on the MAB problems.

\paragraph{Combinatorial semi-bandit} Combinatorial semi-bandits (CSB) is a generalization of MAB, and there are also two types of CSB, i.e., in the adversarial or stochastic settings. Adversarial CSB was introduced in the context of shortest-path problems by \citet{Gyorgy2007}, and later studied extensively \citep{lattimore2018bandit}. There is also a large literature about stochastic CSB \citep{gai2012combinatorial,CWYW16,combes2015combinatorial,kveton2015combinatorial}. Recently, \citet{zimmert2019beating} propose a single algorithm that can achieve the best of both worlds. However, most of the previous works focus on linear reward functions. 
\citet{CWY13,CWYW16} initialize the study of nonlinear CSB. 
\citet{CWY13} consider the problem with $\alpha$-approximation oracle, and \cite{CWYW16} generalize the model with probabilistically triggered arms, which includes the online influence maximization problem. \citet{wang2017improving} further improve the result and remove an exponential term in the regret bound by considering a subclass of CMAB with probabilistically triggered arms, and prove that the online influence maximization belongs to this subclass. \citet{chen2016combinatorial} generalize the model in \cite{CWY13} in another way, 
	and they consider the CMAB problem with a general reward function that is dependent on the distribution of the arms, not only on their means.

\paragraph{Non-stationary bandits} Non-stationary MAB can be viewed as a generalization of the stochastic MAB, where the reward distributions are changing over time. To obtain optimal regret bounds in terms of ${N}$ or ${V}$, most of the studies need to use ${N}$ or ${V}$ as algorithmic parameters, which may not be easy to obtain in practice \citep{garivier2011upper,wei2016tracking,liu2018change,gur2014stochastic,besbes2015nonstationary}. 
Until very recently, an innovative study by \citet{auer2O19adaptively} solves the problem without knowing ${N}$ or ${V}$ in the bandit case and achieves optimal regret. 
Nearly at the same time, \citet{chen2019new} significantly generalizes the previous work by extending it into the non-stationary contextual bandit and also achieves optimal regret without any prior information, but this algorithm is far from practical. The works closest to ours are by \citet{zhou2019near} who also considers non-stationary combinatorial semi-bandits, and by \citet{wang2019aware} who consider the piecewise-stationary cascading bandit. There are also some works considering non-stationary linear bandits \citep{russac2019weighted,kim2019near}, which is a generalization of linear combinatorial bandits. However, the last two studies only achieve optimal bounds when the algorithm knows ${N}$ or ${V}$. 
Although the algorithm in \cite{zhou2019near} is parameter-free, they make other assumptions on the length of the switching period. Moreover, they do not consider the probabilistically triggered arms.


\section{Model}\label{sec:model}

In this section, we introduce our model for the non-stationary combinatorial semi-bandit problem. Our model is derived from \cite{wang2017improving}, which handles nonlinear reward functions, approximate offline oracle, and the probabilistically triggering arms.

We have $m$ base arms $[m]=\{1,2,\ldots, m\}$. At time $t$, the environment samples random outcomes 
    $\bX^{(t)} = (X_1^{(t)},X_2^{(t)},\dots,X_m^{(t)})$ for these arms from a joint distribution ${D}_t\in\mathbb D$.
The sample random variable $X_i^{(t)}$ has support $[0,1]$ for all $i,t$. 
Let $\mu_{i,t} = \E[X_i^{(t)}]$ and we use $\boldsymbol\mu_t = (\mu_{1,t},\mu_{2,t},\dots,\mu_{m,t})$ to denote the mean vector at time $t$. The player does not know ${D}_t$ for any $t$. In round $t \ge 1$, the player selects an action $S_t$ from an action space $\mathbb S$ (could be infinite) based on the feedback from the previous rounds.
When we play action $S_t$ on the environment outcome $\bX^{(t)}$, a random subset of arms $\tau_t\subseteq [m]$ are triggered, and the outcomes of $X_i^{(t)}$ for all $i\in \tau_t$ are observed as the feedback to the player. The player also obtains a nonnegative reward $R(S_t,\bX^{(t)}, \tau_t)$ fully determined by $S_t,\bX^{(t)}$ and $\tau_t$. Our objective is to properly select actions $S_t$'s at each round $t$ based on the previous feedback and maximize the cumulative reward.



For the triggering set $\tau_t$ given the environment outcome $\bX^{(t)}$ and the action $S_t$, we assume that $\tau_t$ is sampled from the distribution ${D}^{trig}(S_t,\bX^{(t)})$, where ${D}^{trig}(S,\bX)$ is the probabilistic triggering function, and it is a probability distribution on the triggered subsets $2^{[m]}$ given the action $S$ and environment outcome $\bX$. Moreover, we use $p_i^{{D},S}$ to denote the probability that action $S$ triggers arm $i$ when the environment triggering distribution is ${D}$. We define $\tilde S^{{D}} = \{i:p_i^{{D},S} > 0\}$ to be the set of arms that can be triggered by action $S$ under distribution ${D}$. 

We assume that $\E[R(S_t,\bX^{(t)}, \tau_t)]$ is a function of $S_t,\boldsymbol\mu_t$, and we use $r_{S}(\boldsymbol\mu) := \E_{\bX}[R(S,\bX,\tau)]$ to denote the expected reward of action $S$ given the mean vector $\boldsymbol\mu$. This assumption is similar to that in \cite{CWYW16,wang2017improving}, and can be satisfied for example when variables $X_i^{(t)}$'s are independent Bernoulli random variables.
Let $\text{opt}_{\boldsymbol\mu_t} := \sup_{S\in\mathbb S}r_{S}(\boldsymbol\mu_t)$ denote the maximum reward in round $t$ given the mean vector $\boldsymbol\mu_t$.

The previous model is similar to that in \cite{wang2017improving}, except that in this paper, we consider the non-stationary setting where ${D}_t$ can change in different rounds. 
We assume that $\{{D}_t\}$ are generated {\em obliviously}, i.e. the generation of ${D}_t$ is completed before the algorithm starts, or equivalently, the generation of ${D}_t$ is independent to the randomness of our algorithm and the randomness of the previous samples $X^{(s)},s< t$. Next, we introduce the measurement of the non-stationarity. In general, there are two measurements of the change of the environment: the first is the number of the swichings ${N}$, and the second is the variation ${V}$ or $\bar{V}$. For any interval $I = [s,s']$, we define the number of switchings on $I$ to be ${N}_{I} := 1 + \sum_{t=s+1}^{s'}\I\{{D}_t \neq {D}_{t-1}\}$, which can be interpreted as the number of stationary segments. As for the variation, we define ${V}_{I} := \sum_{t=s+1}^{s'}||\bmu_t -\bmu_{t-1}||_{\infty}$, which denotes the total change of the mean. 
By the above definitions, we have a simple fact that ${V}_{I} \le {N}_{I}$.
Another similar quantity is the total variation, and the formal definition is given as $\bar{V}_{I} := \sum_{t=s+1}^{s'}||{D}_t -{D}_{t-1}||_{\text{TV}}$, where $||\cdot ||_{\text{TV}}$ denotes the total variation of the distribution. 

${V}$ is a lower bound of $\bar{{V}}$ (see Lemma 9 in \cite{luo2018efficient}). In some cases, $\bar{{V}}$ can be in order $\Theta(T)$ while ${V}$ is a constant (just consider distribution varies but with the same expectation). In non-stationary multi-armed bandits, ${V}$ is more frequently used compared with $\bar{{V}}$ \citep{gur2014stochastic,auer2O19adaptively}. $\bar{{V}}$ is often used in contextual bandits \citep{luo2018efficient,chen2019new}. 

For convenience, we use ${N}$, ${V}$ and $\bar{V}$ to denote ${N}_{[1,T]}$, ${V}_{[1,T]}$ and $\bar{V}_{[1,T]}$ respectively. When we use ${N}$ to measure the non-stationarity, we say that we are considering the switching case. Otherwise, when we are using parameters ${V}$ or $\bar{V}$, we say that we are in the dynamic case. We also define $K = \max_{t,S}|\tilde S^{{D}_t}|$ to be the maximum number of arms that can be triggered by an action in any round. Clearly, we have $K\le m$.

Now we can introduce the measurement of the algorithm. Given an online algorithm $\cA$, we assume that $\cA$ has access to an offline $(\alpha,\beta)$-approximation oracle $\Oracle$, which takes the input $\boldsymbol\mu = (\mu_1,\dots,\mu_m)$ and returns an action $S^{\Oracle}$ such that $\Pr\{r_{\mu}(S^{\Oracle}) \ge \alpha\cdot\text{opt}_{\boldsymbol\mu}\} \ge \beta$. Here, $\alpha$ can be interpreted as the approximation ratio and $\beta$ is the success probability. Based on the $(\alpha,\beta)$-approximation oracle $\Oracle$, we have the following definition of $(\alpha,\beta)$-approximation non-stationary regret:

\begin{definition}[$(\alpha,\beta)$-approximation Non-stationary Regret]
    \label{def: regret}
    The $(\alpha,\beta)$-approximation non-stationary regret for algorithm $\cA$ during the total time horizon $T$ is defined as the following:
     
    \[\text{Reg}^{\cA}_{\alpha,\beta} := \alpha\cdot\beta\cdot\sum_{t=1}^T\text{opt}_{\boldsymbol\mu_t} - \E\left[\sum_{t=1}^T r_{S_t^{\cA}}(\boldsymbol\mu_t)\right],\]
    where $S_t^{\cA}$ is the action selected by algorithm $\cA$ in round $t$.
\end{definition}

Intuitively, the first term $\alpha\cdot\beta\cdot\sum_{t=1}^T\text{opt}_{\boldsymbol\mu_t}$ is the best we can guarantee with the total knowledge of the distributions ${D}_t$ for every round $t$, and the second term is the expected reward selected by our algorithm $\cA$. 

Our regret bounds are in the form $\tilde{O}({N}^{\gamma_1}T^{\gamma_2})$ for the switching measurement and
	$\tilde{O}({V}^{\gamma_3}T^{\gamma_4})$ for the  variation measurement.
Note that if we allow the distributions ${D}_t$ to change arbitrarily in every round, we cannot learn the distribution at all
	and there is no hope to get the non-stationary regret bound ``sub-linear'' in terms of $T$.
This implies that we cannot get regret bounds with $\gamma_1 + \gamma_2 < 1$ or $\gamma_3 + \gamma_4 < 1$, because
	${N}$ and ${V}$ are bounded by $T$ and the above inequalities would lead to sublinear regrets even for arbitrary changes of ${D}_t$.
Thus, the best one can hope for is to achieve regret bounds with $\gamma_1 + \gamma_2 = 1$ or $\gamma_3 + \gamma_4 = 1$.
Indeed, all of our algorithms in the paper achieve such regret bounds.
In this case, as long as ${N}$ or ${V}$ is sublinear in $T$, we would achieve a sublinear regret in $T$.
Moreover, in this case, we also prefer bounds with $\gamma_2$ or $\gamma_4$ as small as possible, because it would lead to better regret bound in $T$ as long as ${N}$ or ${V}$ is sublinear in $T$.
In many cases, our algorithms do achieve the minimum possible $\gamma_2$ or $\gamma_4$, as we discuss later for each algorithm.

We make the following assumptions on the problem instance similar to those in \cite{wang2017improving}, which shows that many important
	CMAB application instances such as influence maximization and combinatorial cascading bandit satisfy these assumptions.

\begin{restatable}[Monotonicity]{assumption}{assmonotonicity}\label{ass:monotonicity}
For any $\bmu$ and $\bmu'$ with $\bmu \le \bmu'$ (dimension-wise), for any action $S$, $r_S(\bmu) \le r_S(\bmu')$.
\end{restatable}

\begin{restatable}[$1$-Norm TPM Bounded Smoothness]{assumption}{asstpmboundedsmooth}\label{ass:tpm-bounded-smoothness}
    For any two distributions ${D},{D}'$ with expectation vectors $\boldsymbol{\mu}$ and $\bmu'$ and any action $S$, we have
    \[|r_S(\bmu) - r_S(\bmu') |\le B\sum_{i\in [m]}p^{{D},S}_i|\bmu_i - \bmu'_i|.\]
\end{restatable}

\section{General Algorithm for Non-stationary CMAB}\label{sec:general}

In this section, we give an algorithm for the general CMAB model defined in Section \ref{sec:model}. We first give the algorithm ($\Cucbsw$) when we know that parameters ${N}$ or ${V}$ that measure the non-stationarity. Then, we show how to combine the $\Cucbsw$ with the Bandit-over-Bandit \cite{cheung2019learning} to get a parameter-free algorithm ($\Cucbbob$).
 
\subsection{Nearly optimal regret when knowing ${N}$ or ${V}$}
In this part, we show our algorithm for the non-stationary CMAB problem when we know the parameter ${N}$ or ${V}$. We apply a standard technique and get a simple algorithm $\Cucbsw$. Although the algorithm is simple and straightforward, the analysis is quite complicated. Our main contribution is the analysis for $\Cucbsw$, especially when we have the approximation oracle and the probabilistic triggering arms. We will first introduce our algorithm $\Cucbsw$, and then state the regret bound and give some discussions on the regret bound and proof sketch.

When we know the parameters ${N}$ or ${V}$, we can apply the sliding window technique to get the result for non-stationary CMAB. The resulting algorithm is simple and included as Algorithm~\ref{alg:non-stationary}: We use CUCB \citep{wang2017improving} in each round, but we only consider the samples in a sliding window with size $w$. 

Generally speaking, in each round, we compute the empirical mean of each arm in a sliding window with size $w$. We also compute the corresponding UCB value for each arm. Then, we use the oracle $\Oracle$ to solve the optimization problem with the UCB value of each arm as input. 

\begin{algorithm}[t]
\caption{Sliding Window CUCB: $\Cucbsw$}
\label{alg:non-stationary}
    \begin{algorithmic}[1]
    \STATE {\bfseries Input:} $m$, Oracle $\Oracle$, time horizon $T$, window size $w \le T$ ($w$ depends on
    	${V}$ or ${N}$, see Theorem~\ref{thm:cucb-sw-prob})
    \FOR{$t=1,2,3,\dots$}
        \STATE $T_{i,t} \leftarrow$ number of time arm $i$ has been triggered in time $\max\{t-w+1,1\},\dots,t-1$.
        \STATE $\hat \mu_{i,t} \leftarrow $ empirical mean of arm $i$ during time $t-w ,\dots,t-1$; ($1$ if not triggered).
        \STATE $\rho_{i,t} \leftarrow \sqrt{\frac{3\ln T}{2T_{i,t}}}$ ($\infty$ if $T_{i,t} = 0$) 
        \STATE $\bar \mu_{i,t} = \min\{\hat \mu_{i,t} + \rho_{i,t},1\}$
        \STATE $S_t\leftarrow\Oracle(\bar \mu_{1,t},\bar \mu_{2,t},\dots,\bar \mu_{m,t})$
        \STATE Play action $S_t$, observe samples from triggered set.
    \ENDFOR
    \end{algorithmic}
\end{algorithm}

To introduce the regret bound for $\Cucbsw$, we need to define the gap in the non-stationary case. Formally, we have the following definition.

\begin{restatable}[Gap]{definition}{defngap}\label{defn:gap}
    For any distribution ${D}$ with mean vector $\bmu$. For each action $S$, we define the gap $\Delta^{{D}}_S := \max\{0,\alpha\cdot\text{opt}_{\bmu} - r_S(\bmu)\}$. For each arm $i$, we define
    \begin{align*}
        & \Delta^{i,t}_{\min} = \inf_{S\in\mathbb S:p^{{D}_t,S}_i > 0,\Delta^{{D}_t}_S > 0}\Delta^{{D}_t}_S, \\
        & \Delta^{i,t}_{\max} = \sup_{S\in\mathbb S:p^{{D}_t,S}_i > 0,\Delta^{{D}_t}_S > 0}\Delta^{{D}_t}_S.
    \end{align*}
    We define $\Delta_{\min}^i = +\infty$ and $\Delta_{\max}^i = 0$ if they are not properly defined by the above definitions. Furthermore, we define $\Delta^{i}_{\min} := \min_{t\le T}\Delta^{i,t}_{\min}$, $\Delta^{i}_{\max} := \max_{t\le T}\Delta^{i,t}_{\max}$ as the minimum and maximum gap for each arm.
\end{restatable}

In the above definition, the gap $\Delta_{\min}^{i,t},\Delta_{\max}^{i,t}$ for a fixed arm $i$ and a fixed time is similar to the definition of gap in \cite{wang2017improving}. However, their definition is based on a single distribution ${D}$, and in our setting, we need to generalize the definition from stationary case to dynamic case where we need to take several distributions into account. Our generalization from the stationary to the dynamic case is similar to the generalization in \cite{garivier2011upper}, which takes the minimum of the gap in each round. With the above definition, we have the following regret bound.

\begin{restatable}[Regret for $\Cucbsw$]{theorem}{thmcucbswprob}\label{thm:cucb-sw-prob}
Choosing the length of the sliding window to be $w = \min\left\{\sqrt{\frac{T}{{V}}},T\right\}$, we have the following distribution-dependent bound,
\[\text{Reg}_{\alpha,\beta} = \tilde{O}\left(\sum_{i\in [m]}\frac{K\sqrt{{V} T}}{\Delta^i_{\min}} + \sum_{i\in [m]}\frac{K}{\Delta^i_{\min}} + mK \right).\]
If we choose the length of the sliding window to be $w = \min\left\{m^{1/3}T^{2/3}K^{-1/3}{V}^{-2/3},T\right\}$, we have the following distribution-independent bound,
\[\text{Reg}_{\alpha,\beta} = \tilde{O}\left((m{V})^{1/3}(KT)^{2/3} + \sqrt{mKT} + mK\right).\]
\end{restatable}

Note that since we have ${V} \le {N}$, we can change the parameter from ${V}$ to ${N}$ in both of the regret bounds. We first look at the distribution-dependent bound. Unlike the distribution-dependent bound for the stationary MAB problem, the distribution-dependent bound here has order $\tilde{O}(\sqrt{T})$. However, the $\tilde{O}(\sqrt{T})$ term is unavoidable, since the distribution-dependent bound is lower bounded by $\Omega(\sqrt{T})$ \citep{garivier2011upper}. Although \cite{garivier2011upper} only prove the lower bound in the switching case, it also applies to the dynamic case since the switching case is a special case of the dynamic case. In this way, our distribution-dependent bound is nearly optimal in both cases in terms of ${V}$, ${N}$, and $T$.


As for the distribution-independent bound, the leading term in the dynamic case is $(m{V})^{1/3}(KT)^{2/3}$. This term is optimal in terms of ${V}$ and $T$ and we cannot further improve the exponential term. The second term $\sqrt{mKT}$ is also necessary, since this term will be the leading term when ${V}$ is very small, and the non-stationary CMAB degenerates to the original stationary CMAB problem. It is well known that $\sqrt{mT}$ is the lower bound for stationary MAB problem with $m$ arms, so the second term is also optimal. In this way, our distribution-independent bound is nearly optimal in the dynamic case. However, the bound in the switching case is not tight. Our upper bound is ${N}^{1/3}T^{2/3}$ but the current upper and lower bound for non-stationary MAB is $\sqrt{{N} T}$ \citep{auer2O19adaptively,chen2019new}. Designing nearly optimal regret bound for the switching case is left as future work.


The readers may find that the window lengths are not the same in the theorem for distribution-dependent/independent bounds. The different lengths are crucial to get optimal bounds since we optimize the regret bounds by the window length.

The readers may also be curious about the distribution change of the triggering probability. Note that in the model part (Section \ref{sec:model}), we do not explicitly define the distribution change of the triggering probability. However, the change of the triggering probability can change the reward a lot. The intuition is that, although we do not define the change of the triggering probability, the triggering probability is ``induced'' by the distribution of the outcome of each arm (e.g., the triggering of an edge in influence maximization problem is totally determined by the propagation probability of each arm). Besides, because of the TPM bounded smoothness (Assumption \ref{ass:tpm-bounded-smoothness}), the regret can also be bounded. In this way, we transfer the regret due to the change of the triggering probability to the regret due to the change of the arm outcome distribution, which is also the key challenge in our proof.

Now we briefly show our proof idea to handle the probabilistically triggered arms. Like the proof in \citet{wang2017improving}, we first partition the action-distribution pair $S^{D}$ into groups where $G_{i,j} = \{S^{{D}}\in\mathbb S\times \mathbb D| 2^{-j} < p_i^{{D},S}\le 2^{-j+1}\}$. 
Generally speaking, $G_{i,j}$ includes the action-distribution pairs that $S$ triggers arm $i$ under distribution $D$ with probability around $2^{-j}$. Then, we define another quantity $N_{i,j,t}$ for arm $i$ that may be triggered in group $G_{i,j}$, and it will count at time $s$ in the sliding window ends at $t$ if $2^{-j} < p_i^{{D}_s,S_s}\le 2^{-j+1}$. 
Intuitively, the expected number of triggers of arm $i$ during the sliding window can be upper-bounded by $2^{-j+1}N_{i,j,t}$ and lower bounded by $2^{-j}N_{i,j,t}$. Formally, we have the following definition for $N_{i,j,t}$.
\begin{restatable}[Counter]{definition}{defcounter}\label{defn:counter}
    Given the sliding window size $w$ of the algorithm, in a run of the algorithm, we define the counter $N_{i,j,t}$ as the following number
    \[N_{i,j,t} := \sum_{s=\max\{t-w+1,0\}}^t \I\left\{2^{-j} < p^{D_s,S_s}_i\le 2^{-j+1}\right\}.\]
\end{restatable}
The first step is to relate the $(\alpha,\beta)$-approximation non-stationary regret with the quantities $N_{i,j,t}$. All the terms related to the triggering probability can be converted to $N_{i,j,t}$. Next, we bound the formula with $N_{i,j,t}$. We show that the formula is non-increasing with respect to $N_{i,j,t}$, and we find another instance $N'$ such that $N'_{i,j,t} \le N_{i,j,t}$. The formula with $N'_{i,j,t}$ is easier to get regret upper bound and we use that quantity to bridge between the regret and the upper bound.

\subsection{Parameter-free algorithm}

\begin{algorithm}[t]
\caption{CUCB with Bandit over Bandit: $\Cucbbob$}
\label{alg:cucb-bob}
\begin{algorithmic}[1]
    \STATE {\bfseries Input:} Total time horizon $T$, Block size $L$, Parameters $R = R_2 - R_1$ where $R_1 \le r_S(\mathbf 0) \le r_S(\mathbf 1) \le R_2$.
    \STATE Suppose $2^k \le L < 2^{k+1}$. Set up an EXP3.P that has $k+1$ arms. Arm $i$ corresponds to window size $2^i$.
    \FOR{$\ell = 1,2,\dots,\lceil \frac{T}{L}\rceil$}
        \STATE Set up an algorithm $\Cucbsw$ for block $\ell$, choosing the window size according to EXP3.P.
        \FOR{$t = (\ell-1)L + 1,\dots,\min\{\ell L, T\}$}
            \STATE Act according to the $\Cucbsw$ in block $\ell$.
        \ENDFOR
        \STATE $R(\ell)$ is the total reward in block $\ell$.
        \STATE Pass $\frac{R(\ell)-R_1}{R}$ to EXP3.P. // Normalize to $[0,1]$
    \ENDFOR
\end{algorithmic}
\end{algorithm}

In this section, we introduce our parameter-free algorithm for the non-stationary CMAB problem. We combined the Bandit-over-Bandit technique \citep{cheung2019learning} with the previous sliding window CUCB algorithm ($\Cucbsw$), and design our parameter-free algorithm $\Cucbbob$ for general non-stationary CMAB problem.

Generally speaking, the Bandit-over-Bandit technique can be summarized as follow: We first divide the total time horizon $T$ into several segments where each segment has length $L$ (the last segment may not). Although we do not know the non-stationary parameters ${N}$ or ${V}$, 
we can guess ${N}$ or ${V}$, or other parameters used by the algorithm when we know the parameters ${N}$ or ${V}$. For example, we can guess the length of the sliding window of $\Cucbsw$. For two different blocks, we may run the algorithm with different guessing parameters. 
However, random guessing cannot have a good performance guarantee, and we use a ``master bandit algorithm'' to control our guessing. 
Whenever we complete the algorithm for a block with some guessing parameter, we feed the total reward in this block to the master bandit algorithm, and the master bandit algorithm will return us the parameter used in the next block.

In our non-stationary CMAB case, we combine the Bandit-over-Bandit technique with the previous sliding window algorithm $\Cucbsw$. First, we assume that we have EXP3.P algorithm for the master bandit \citep{bubeck2012regret}, which is a variant of the original EXP3 algorithm. We choose EXP3.P because it is easier to derive the regret bound since the regret of EXP3.P is bounded, while the original EXP3 only has pseudo-regret bound. Furthermore, we also assume that there exists parameters $R = R_2 - R_1$ where $R_1 \le r_S(\mathbf 0) \le r_S(\mathbf 1) \le R_2$. This assumption aims to bound the optimal value in each round. Without this assumption, the reward in each round may be too large. Our algorithm takes $L$ as input, which denotes the length of each block, and its proper value is given in Theorem~\ref{thm:cucb-bob}. We discretize the possible sliding window size in an exponential way: The possible window size are $1,2,4,\dots,2^k$ where $2^k \le L < 2^{k+1}$. There are ${O}(\log_2 L)$ number of possible window sizes in total. Then in each block, we run $\Cucbsw$ with some window size, and we control the window size by the master EXP3.P algorithm. The only thing left is that we need to feed the reward to the EXP3.P algorithm. Here we assume that the reward in each round is bounded, and we can compute the total reward in each block and normalize it into $[0,1]$. Please see Algorithm \ref{alg:cucb-bob} for more details.

\begin{restatable}[]{theorem}{thmcucbbob}\label{thm:cucb-bob}
    Suppose that there exist $R_1,R_2$ such that $R_1 \le r_S(\mathbf 0) \le r_S(\mathbf 1) \le R_2$ for any $S\in\mathbb S$ and $R = R_2 - R_1$. Choosing $L = \sqrt{mKT} / R$, we have the following distribution-independent regret bound for $\text{Reg}_{\alpha,\beta}$,
    \[\tilde{O}\left((m{V})^{\frac{1}{3}}(KT)^{\frac{2}{3}} + \sqrt{R}(mK)^{\frac{1}{4}}T^{\frac{3}{4}} +R\sqrt{mKT}\right).\]
    Choosing $L = K^{2/3}T^{1/3}$, we have the following distribution-dependent regret bound
    \[\tilde{O}\left(K\sqrt{\sum_{i\in [m]}\frac{TV}{\Delta^i_{\min}}} + \sum_{i\in [m]}\frac{K^{\frac{1}{3}}T^{\frac{2}{3}}}{\Delta^i_{\min}} + RK^{\frac{1}{3}}T^{\frac{2}{3}}\right).\]
\end{restatable}

In this theorem, we do not need different window lengths, since the algorithm chooses for us. However, we need different block sizes. The difference aims to optimize the sublinear term in $T$ ($T^{3/4}$ for distribution-independent and $T^{2/3}$ for distribution-dependent). We can choose $L = \sqrt{T}$ in both cases, then the sublinear term may be worse, and we may also lose some factors in terms of $m,K$.

Note that since ${V}\le{N}$, we can also replace ${V}$ by ${N}$ in the above regret bounds. First let's focus on the distribution-independent bound. As discussed in the previous section, $(m{V})^{\frac{1}{3}}(KT)^{\frac{2}{3}}$ is nearly optimal and we can not improve this term in terms of $m,{V},T$. The last term $R\sqrt{mKT}$ is also nearly optimal. However, the term $\sqrt{R}(mK)^{\frac{1}{4}}T^{\frac{3}{4}}$ is not optimal. Nontheless, this term is sublinear and the total regret is also sublinear in $T$ as long as ${V} < c T^{\gamma}$ for some $\gamma < 1$. When we change ${V}$ into ${N}$, as discussed before, there is a gap between the bound $(m{N})^{1/3}(KT)^{2/3}$ and the existed lower bound $\sqrt{m{N} T}$. Despite of this, the total regret bound is sublinear in $T$ if ${N} < c T^{\gamma}$ for some $\gamma < 1$.

As for the distribution-dependent bound, the first term is nearly optimal both in the dynamic case (measured by ${V}$) and in the switching case ${N}$. The sub-optimality comes from the second term $\sum_{i\in [m]}\frac{K^{\frac{1}{3}}T^{\frac{2}{3}}}{\Delta^i_{\min}}$. Despite this, the regret bound is ``sublinear'' and it is nearly optimal when ${N}$ or ${V}$ are large. Also, note that the first term is better than the term for fixed window size because we are guessing the best window size, which can take the gaps into account. However, in the fixed window size scenario, the gaps are unknown parameters and we can only optimize through ${V}$.

Next, we briefly show the intuition of the proof. We first have the following theorem for the performance guarantee of EXP3.P algorithm \citep{bubeck2012regret}. 

\begin{restatable}[Regret of EXP3.P]{proposition}{propregretexp}\label{prop:exp3p-regret}
    Suppose that the reward of each arm in each round is bounded by $0\le r_{i,t}\le R'$, the number of arms is $K'$, and the total time horizon is $T'$. The expected regret of EXP3.P algorithm is bounded by ${O}(R'\sqrt{K'T'\log K'})$.
\end{restatable}

The general idea of the proof is to decompose the $(\alpha,\beta)$-regret of algorithm $\Cucbbob$ into two parts: The first part is the regret of the algorithm $\Cucbsw$ with the best size of sliding window; the second part is the difference between the reward of $\Cucbsw$ with best sliding window and the reward of $\Cucbbob$. The bound for the first part is given in the previous section, and we want each block to be large. Otherwise, the ``best'' window size cannot be reached. The second part of the regret can be bounded by the EXP3.P algorithm. If we select the length of each block as $L$, then each reward is at order $L$. There are $\log_2 T$ arms in total and the time horizon for the EXP3.P algorithm is $\frac{T}{L}$. In this way, the second term is at order $\tilde{O}(L\sqrt{T/L}) = \tilde{O}(\sqrt{TL})$, and we want $L$ to be small for the second part. Optimizing for $L$, we can get the bound in Theorem \ref{thm:cucb-bob}.

%

There are two aspects that make designing a nearly optimal parameter-free algorithm hard. The first is the combinatorial structure of the offline problem: If we want to explore a single base arm, we may afford a large regret, and if we want to eliminate a base arm, we may affect a lot of actions. The second is the approximation oracle: It is hard to detect the non-stationarity through the reward of each round since the rewards are not accurate. A very small change in the input of the oracle may lead to a huge difference in the output of the oracle.
In the next section, we show that in the restricted case of linear CMAB with exact offline oracle, we do achieve near-optimal regret.

\section{Nearly Optimal Algorithm in Special Case}
\label{sec:special}

\begin{algorithm*}
   \caption{\textsc{Ada-LCMAB}}
   \label{alg: LCMAB}
\begin{algorithmic}[1]
   \STATE {\bfseries Input:} confidence $\delta$, time horizon $T$, action space $\mb{S}$
   \STATE {\bfseries Definition:} $\nu_j=\sqrt{\frac{C_0}{m2^jL}}$, where $C_0=\ln\left(\frac{8T^3|\mb{S}|^2}{\delta}\right)$, $L=\lceil 4mC_0 \rceil, \mc{B}_{(i,j)}:=[\iota_i, \iota_i+2^jL-1]$. 
   \STATE {\bfseries Initialize:} $t=1, i=1$
   \STATE $ \iota_i \leftarrow t$ \label{alg_line: init}
   \FOR{$j=0,1,2,\dots$}
   \STATE If $j=0$, set $Q_{(i,j)}$ as an arbitrary distribution over $\mb{S}$; otherwise, let $(\bm{q}_{(i,j)}^{\nu_j}, Q_{(i,j)}^{\nu_j})$ be the associated solution and distribution of equation (\ref{eq: FTRL}) with inputs ${I}=\mc{B}_{(i,j-1)}$ and $\nu = \nu_j$
   \STATE $\mc{E} \leftarrow \emptyset$
   \WHILE{$t \leqslant \iota_i+2^jL-1$}
   \STATE Draw $\mathrm{REP} $ $\sim $ $\mathrm{Bernoulli}\left(\frac{1}{L}\times2^{-j/2}\times \sum_{k=0}^{j-1}2^{-k/2}\right)$ 
   \IF{$\mr{REP}=1$}
   \STATE Sample $n$ from $\{0,\dots, j-1\}$ s.t. $\Pr[n=b]\propto 2^{-b/2}$
   \STATE $\mc{E} \leftarrow \mc{E} \cup \{(n,[t,t+2^nL-1])\}$
   \ENDIF
   \STATE Let $\cN_t:=\{n|\exists {I} \text{ such that } t \in {I} \text{ and } (n, {I})\in \mc{E}\}$
   \STATE If $\cN_t$ is empty, play $S_t \sim Q_{(i,j)}^{\nu_j}$; otherwise, sample $n \sim \text{Uniform}(\cN_t)$, and play $S_t \sim Q^{\nu_n}_{(i,n)}$ \label{alg: LCMAB strategy}
   \STATE Receive $\{X_i^t|i \in S_t\}$ and calculate $\hat{{\bmu}}_t$ according to equation (\ref{eq: importance weight})
   \FOR{$(n,[s,s']) \in \mc{E}$}
   \IF{$s'=t$ and \textsc{EndOfReplayTest}$(i,j,n,[s,t])=Fail$}
   \STATE $t \leftarrow t+1, i\leftarrow i+1$ and return to Line \ref{alg_line: init}
   \ENDIF
   \ENDFOR
   \IF{$t=\iota_i+2^jL-1$ and \textsc{EnfOfBlockTest}$(i,j)=Fail$}
   \STATE $t\leftarrow t+1, i\leftarrow i+1$ and return to Line \ref{alg_line: init}
   \ENDIF
   \ENDWHILE
   \ENDFOR
\end{algorithmic}
\begin{algorithmic}
    \STATE {\bfseries Procedure:} \textsc{EndOfReplayTest}($i,j,n,\mc{A}$):
    \STATE Return \textit{Fail} if there exists $S \in \mb{S}$ such that any of the following inequalities holds: 
    \begin{align}
        \widehat{\mr{Reg}}_{\mc{A}}(S) - 4 \widehat{\mr{Reg}}_{\mc{B}(i,j-1)}(S) \geqslant 34 mK \nu_n \log T \label{check: replay1}\\
        \widehat{\mr{Reg}}_{\mc{B}(i,j-1)}(S) - 4\widehat{\mr{Reg}}_{\mc{A}}(S) \geqslant 34 mK \nu_n \log T \label{check: replay2}
    \end{align}
\end{algorithmic}
\begin{algorithmic}
    \STATE {\bfseries Procedure:} \textsc{EndOfBlockTest}($i,j$):
    \STATE Return \textit{Fail} if there exists $k \in \{0,1,\dots, j-1\}$ and $S \in \mb{N}$ such that any of the following inequalities holds:
    \begin{align}
        \widehat{\mr{Reg}}_{\mc{B}(i,j)}(S) - 4 \widehat{\mr{Reg}}_{\mc{B}(i,k)}(S) \geqslant 20 mK \nu_k \log T \label{check: block1}\\
        \widehat{\mr{Reg}}_{\mc{B}(i,k)}(S) - 4\widehat{\mr{Reg}}_{\mc{B}(i,j)}(S) \geqslant 20 mK \nu_k \log T \label{check: block2}
    \end{align}
\end{algorithmic}
\end{algorithm*}

In this section, we propose a different algorithm that achieves nearly optimal guarantee for non-stationary linear CMAB \textit{without} any prior information. Our algorithm is based on \textsc{Ada-ILTCB}$^+$ of \cite{chen2019new} designed for non-stationary contextual bandits, but adapted to Linear CMAB with exact oracles (i.e. $\alpha=\beta=1$). In \textsc{Ada-ILTCB}$^+$, the algorithm works on scheduled blocks with exponentially increasing length. In each block, since there is no restart in \textit{previous blocks}, it is safe to adopt a previously learned strategy as the underlying distribution does not change. To detect non-stationarity, the algorithm randomly triggers some replay phases with different granularities and compares the performance of each policy over these intervals. If underlying distribution changes, which will cause a gap between performances over different intervals for the same policy, the algorithm will then detect it with high probability, reset all parameters and restart.  

Compared with contextual bandits, which only plays over $m$ arms, the size of action space $\mb{S}$ in CMAB can be exponentially large in terms of $m$. Though each action in CMAB can be regarded as a policy and a base arm in contextual bandits setting, a straightforward implementation of $\textsc{Ada-ILTCB}^+$ \citep{chen2019new} will cause a regret depends on $|\mb{S}|$, which is unsatisfactory. To deal with this issue, we make full use of semi-bandit information, and adopt classic importance weight estimator for underlying unknown linear reward ${\bmu}_t$ \citep{audibert2014regret,zimmert2019beating}. In detail, we calculate a distribution $Q$ over the action space $\mb{S}$ at each round, and play a random action $S$ drawn from $Q$. For the expectation $\bm{q}$ associated with distribution $Q$, apparently for any $i\in [m]$, $\hat{\mu}_{i}=\frac{X_i}{q_{i}}\mb{I}(i\in S)$ constitutes an unbiased estimation of ${\bmu}$ at position $i$, where $\bm{X}$ is a random observation with mean ${\bmu}$. For some notations,  we use $\one_S$ to represent corresponding binary $m$-dimensional vector of a super arm $S$, and $\mb{I}_{\{\cdot\}}$ denotes the indicator function of some event. Given an interval ${I}$, denote $\hat{{\bmu}}_{{I}}:=\sum_{t\in{I}} \hat{{\bmu}}_t /|{I}|$, $\widehat{\mr{Reg}}_{{I}}(S):=\hat{{\bmu}}_{{I}}^\top \one_{\hat{S}_{{I}}} - \hat{{\bmu}}_{{I}}^\top \one_S$ as the empirical mean and empirical regret in this interval, where $\hat{{\mu}}_t$ is the empirical estimation of ${\mu}_t$ at time $t$, $\hat{S}_{{I}}:=\argmax_{S\in\mb{S}}\hat{{\bmu}}_{{I}}^\top \one_S$. $\mr{Conv}(\mb{S})$ represents the convex hull of $\mb{S}$ in the vector space, and define $\mr{Conv}(\mb{S})_\nu = \{\forall \bm{x} \in \mr{Conv}(\mb{S}), s.t. \forall i \in [m], x_i \geqslant \nu\}$. Given a distribution $Q$ over $\mr{Conv}(\mb{S})_\nu$, denote its expectation as $\bm{q}:= \mb{E}_{S\sim Q} \one_S$ and define $\mr{Var}(Q, S):=\sum_{i \in S} 1/q_i$.

Similar to contextual bandits, we show that the solution to Follow The Regularized Leader (FTRL) with log-barrier for CMAB also satisfies some nice properties as stated in the following lemma. Besides, instead of using Frank-Wolfe or other similar algorithm adopted in stationary or non-stationary contextual bandits \citep{agarwal2014taming,chen2019new}, which is unavoidable as we deal with general non-linear function, FTRL for linear combinatorial semi-bandits can be solved efficiently with time complexity in polynomial order of $m$ and $T$ when $\mr{Conv}(\mb{S})$ can be described by a polynomial number of constraints \citep{zimmert2019beating}.  

\begin{restatable}{lemma}{lemftrl}\label{lem: FTRL}
For any time interval ${I}$, its empirical reward estimation $\hat{{\mu}}_{{I}}$, and exploration parameter $\nu>0$, let $\bm{q}^\nu_{{I}}$ be the solution to following optimization problem (\ref{eq: FTRL}) with constant $C=100$:
\begin{equation}
    \label{eq: FTRL}
    \bm{q}^\nu_{{I}} = \argmax_{\bm{q}\in\mr{Conv}(\mb{S})_\nu} \inner{\bm{q}, \hat{{\bmu}}_{{I}}} + C \nu\sum_{i=1}^m \log q_i 
\end{equation}
Let $Q^\nu_{{I}}$ be the distribution over $\mb{N}$ such that $\mb{E}_{S\sim Q^\nu_{{I}}}[\one_S] = \bm{q}^\nu_{{I}}$, then there is 
 
\begin{flalign}
    \sum_{S\in\mb{S}} Q^\nu_{{I}}(S) \widehat{\mr{Reg}}_{{I}}(S) \leqslant Cm\nu \label{ineq: small regret}\\ 
    \forall S \in \mb{S}, ~ \mr{Var}(Q^\nu_{{I}}, S) \leqslant m+\frac{\widehat{\mr{Reg}}_{{I}}(S)}{C\nu} \label{ineq: small variance}
\end{flalign}
\end{restatable}
 
With above FTRL oracle, our full implementation for non-stationary linear combinatorial semi-bandits is detailed in Algorithm \ref{alg: LCMAB}. According to Line \ref{alg: LCMAB strategy} and our estimation method, we know the expectation vector of our sampling strategy and estimated vector $\hat{{\mu}}_t$ are calculated as: 
 
\begin{align}
    \bm{q}_t =& \bm{q}^{\nu_j}_{(i,j)} \mb{I}_{N_t = \emptyset} + \frac{1}{|N_t|}\sum_{n\in N_t}  \bm{q}^{\nu_n}_{(i,n)} \mb{I}_{N_t \neq \emptyset} \\
    \hat{{\mu}}_{t,i} =& \frac{X_i^t}{q_{t,i}}\mb{I}(i\in S_t), \quad \forall i \in [m] \label{eq: importance weight}
\end{align} 
 

For two procedures of non-stationary test in Algorithm \ref{alg: LCMAB}, as we consider linear CMAB and have an exact oracle, which is equivalent to an Empirical Risk Minimization oracle (i.e. giving empirical loss function returns corresponding best super arm), we can use the same technique as in \citet{chen2019new} to solve two procedures with only six oracle calls. 

Since a super arm is pulled at each round for CMAB, it will cause larger variance compared with pulling a single arm in contextual bandits, which requires some additional analysis. Besides, as there is no context in CMAB, we can obtain much smaller constants in \textsc{Ada-LCMAB} compared with original $\textsc{Ada-ILTCB}^+$ \citep{chen2019new}. Now, we state the theoretical guarantee of our proposed algorithm for non-stationary linear CMAB.

\begin{restatable}{theorem}{thrLCMAB}\label{thr: LCMAB}
Algorithm \ref{alg: LCMAB} guarantees $\text{Reg}^{\cA}_{1,1}$ is upper bounded by
\begin{equation*}
    \tilde{O}\left( \min \left\{\sqrt{mK^2NT}, \sqrt{mK^2T}+K(m\bar{V})^{\frac{1}{3}}T^{\frac{2}{3}}\right\} \right).
\end{equation*}
\end{restatable}


Note that in the previous theorem, the regret upper bound is nearly optimal in terms of $m,{N},T$ and $m,\bar{V},T$. Because we know that the regret lower bound for stationary MAB problem is $\Omega(\sqrt{mT})$ with $m$ arms, we can construct special cases to achieve regret lower bound $\Omega(\sqrt{m{N} T})$ in the switching case, and $\Omega((m\bar{V})^{1/3}T^{2/3})$ in the dynamic case. The technique is standard and we refer \cite{gur2014stochastic} for more details on the construction of the special cases. However, the dependent on $K$ may not be tight, 
and we left it as a future work item to tighten the dependency on $K$.

Another possible improvement is to change the measurement $\bar{V}$ in the regret bound into ${V}$. Although in the special cases we construct for the lower bound, ${V}$ and $\bar{V}$ are at the same order,
in other cases ${V}$ is just a lower bound on $\bar{V}$. Improving $\bar{V}$ into ${V}$ is also left as future work.



\section{Conclusion and Further Works}
In this paper, we study combinatorial semi-bandit (CSB) in the non-stationary environment, an extension of classic multi-armed bandits (MAB). 
Our CSB setting also allows non-linear reward function, probabilistically triggering behavior, and approximation oracle, which make our problem more difficult compared with non-stationary MAB or linear bandits. We first propose an optimal algorithm that achieves $\tilde{O}(m\sqrt{{N} T}/\Delta_{\min})$ distribution-dependent regret in the switching case and $\tilde{O}({V}^{1/3}T^{2/3})$ distribution-independent regret in the dynamic case, when ${N}$ or ${V}$ is known.
To get rid of parameter ${N}$ or ${V}$, We further design a parameter-free version with regret bound $\tilde{O}(\sqrt{m{N} T/\Delta_{\min}}+T^{2/3}/\Delta_{\min})$ and $\tilde{O}({V}^{1/3}T^{2/3}+T^{3/4})$ respectively. 
For a special case where the reward function is linear and we have an exact oracle, we design an optimal parameter-free algorithm that achieves nearly optimal regret both in the switching case and in the dynamic case.

As mentioned in Section \ref{sec:general} and \ref{sec:special}, there are several interesting further works. The most important one is to design an optimal parameter-free algorithm for our general CSB. Second, we mainly focus on the dependence on $N$, $V$ or $\bar V$, and $T$, How to improve the dependence on $K$ is a meaningful direction. Finally, a tight lower bound in terms of all the above parameters is necessary for a full understanding of this problem.

\section*{Acknowledgement}
This work was supported by Key-Area Research and Development Program of Guangdong Province (No. 2019B121204008)], National Key R\&D Program of China (2018YFB1402600), BJNSF (L172037) and Beijing Academy of Artificial Intelligence.

\bibliography{chen_330}

\begin{thebibliography}{35}
\providecommand{\natexlab}[1]{#1}
\providecommand{\url}[1]{\texttt{#1}}
\expandafter\ifx\csname urlstyle\endcsname\relax
  \providecommand{\doi}[1]{doi: #1}\else
  \providecommand{\doi}{doi: \begingroup \urlstyle{rm}\Url}\fi

\bibitem[Agarwal et~al.(2014)Agarwal, Hsu, Kale, Langford, Li, and
  Schapire]{agarwal2014taming}
Alekh Agarwal, Daniel Hsu, Satyen Kale, John Langford, Lihong Li, and Robert
  Schapire.
\newblock Taming the monster: A fast and simple algorithm for contextual
  bandits.
\newblock In \emph{International Conference on Machine Learning}, pages
  1638--1646, 2014.

\bibitem[Audibert et~al.(2014)Audibert, Bubeck, and Lugosi]{audibert2014regret}
Jean-Yves Audibert, S{\'e}bastien Bubeck, and G{\'a}bor Lugosi.
\newblock Regret in online combinatorial optimization.
\newblock \emph{Mathematics of Operations Research}, 39\penalty0 (1):\penalty0
  31--45, 2014.

\bibitem[Auer et~al.(2002{\natexlab{a}})Auer, Cesa-Bianchi, and
  Fischer]{auer2002finite}
Peter Auer, Nicol{\`o} Cesa-Bianchi, and Paul Fischer.
\newblock Finite-time analysis of the multiarmed bandit problem.
\newblock \emph{Machine Learning}, 47\penalty0 (2-3):\penalty0 235--256,
  2002{\natexlab{a}}.

\bibitem[Auer et~al.(2002{\natexlab{b}})Auer, Cesa-Bianchi, Freund, and
  Schapire]{auer2002nonstochastic}
Peter Auer, Nicol{\`o} Cesa-Bianchi, Yoav Freund, and Robert~E. Schapire.
\newblock The nonstochastic multiarmed bandit problem.
\newblock \emph{SIAM J. Comput.}, 32\penalty0 (1):\penalty0 48--77,
  2002{\natexlab{b}}.

\bibitem[Auer et~al.(2019)Auer, Gajane, and Ortner]{auer2O19adaptively}
Peter Auer, Pratik Gajane, and Ronald Ortner.
\newblock Adaptively tracking the best bandit arm with an unknown number of
  distribution changes.
\newblock In \emph{Conference on Learning Theory, {COLT} 2019, 25-28 June 2019,
  Phoenix, AZ, {USA}}, pages 138--158, 2019.

\bibitem[Besbes et~al.(2015)Besbes, Gur, and Zeevi]{besbes2015nonstationary}
Omar Besbes, Yonatan Gur, and Assaf~J. Zeevi.
\newblock Non-stationary stochastic optimization.
\newblock \emph{Operations Research}, 63\penalty0 (5):\penalty0 1227--1244,
  2015.
\newblock \doi{10.1287/opre.2015.1408}.

\bibitem[Bubeck and Cesa{-}Bianchi(2012)]{bubeck2012regret}
S{\'{e}}bastien Bubeck and Nicol{\`{o}} Cesa{-}Bianchi.
\newblock Regret analysis of stochastic and nonstochastic multi-armed bandit
  problems.
\newblock \emph{Foundations and Trends in Machine Learning}, 5\penalty0
  (1):\penalty0 1--122, 2012.
\newblock \doi{10.1561/2200000024}.

\bibitem[Chen et~al.(2013)Chen, Wang, and Yuan]{CWY13}
Wei Chen, Yajun Wang, and Yang Yuan.
\newblock Combinatorial multi-armed bandit: General framework, results, and
  applications.
\newblock In \emph{Proceedings of the 30th International Conference on Machine
  Learning (ICML)}, 2013.

\bibitem[Chen et~al.(2016{\natexlab{a}})Chen, Hu, Li, Li, Liu, and
  Lu]{chen2016combinatorial}
Wei Chen, Wei Hu, Fu~Li, Jian Li, Yu~Liu, and Pinyan Lu.
\newblock Combinatorial multi-armed bandit with general reward functions.
\newblock In \emph{Advances in Neural Information Processing Systems}, pages
  1659--1667, 2016{\natexlab{a}}.

\bibitem[Chen et~al.(2016{\natexlab{b}})Chen, Wang, Yuan, and Wang]{CWYW16}
Wei Chen, Yajun Wang, Yang Yuan, and Qinshi Wang.
\newblock Combinatorial multi-armed bandit and its extension to
  probabilistically triggered arms.
\newblock \emph{Journal of Machine Learning Research}, 17\penalty0
  (50):\penalty0 1--33, 2016{\natexlab{b}}.
\newblock A preliminary version appeared as Chen, Wang, and Yuan,
  ``combinatorial multi-armed bandit: General framework, results and
  applications'', ICML'2013.

\bibitem[Chen et~al.(2019)Chen, Lee, Luo, and Wei]{chen2019new}
Yifang Chen, Chung-Wei Lee, Haipeng Luo, and Chen-Yu Wei.
\newblock A new algorithm for non-stationary contextual bandits: Efficient,
  optimal and parameter-free.
\newblock In Alina Beygelzimer and Daniel Hsu, editors, \emph{Proceedings of
  the Thirty-Second Conference on Learning Theory}, volume~99 of
  \emph{Proceedings of Machine Learning Research}, pages 696--726, Phoenix,
  USA, 25--28 Jun 2019. PMLR.

\bibitem[Cheung et~al.(2019)Cheung, Simchi-Levi, and Zhu]{cheung2019learning}
Wang~Chi Cheung, David Simchi-Levi, and Ruihao Zhu.
\newblock Learning to optimize under non-stationarity.
\newblock In Kamalika Chaudhuri and Masashi Sugiyama, editors,
  \emph{Proceedings of Machine Learning Research}, volume~89 of
  \emph{Proceedings of Machine Learning Research}, pages 1079--1087. PMLR,
  16--18 Apr 2019.

\bibitem[Combes et~al.(2015)Combes, Shahi, Proutiere,
  et~al.]{combes2015combinatorial}
Richard Combes, Mohammad Sadegh Talebi~Mazraeh Shahi, Alexandre Proutiere,
  et~al.
\newblock Combinatorial bandits revisited.
\newblock In \emph{Advances in Neural Information Processing Systems}, pages
  2107--2115, 2015.

\bibitem[Gai et~al.(2012)Gai, Krishnamachari, and Jain]{gai2012combinatorial}
Yi~Gai, Bhaskar Krishnamachari, and Rahul Jain.
\newblock Combinatorial network optimization with unknown variables:
  Multi-armed bandits with linear rewards and individual observations.
\newblock \emph{IEEE/ACM Transactions on Networking}, 20\penalty0 (5):\penalty0
  1466--1478, 2012.

\bibitem[Garivier and Moulines(2011)]{garivier2011upper}
Aur{\'{e}}lien Garivier and Eric Moulines.
\newblock On upper-confidence bound policies for switching bandit problems.
\newblock In \emph{Algorithmic Learning Theory - 22nd International Conference,
  {ALT} 2011, Espoo, Finland, October 5-7, 2011. Proceedings}, pages 174--188,
  2011.
\newblock \doi{10.1007/978-3-642-24412-4\_16}.

\bibitem[Gur et~al.(2014)Gur, Zeevi, and Besbes]{gur2014stochastic}
Yonatan Gur, Assaf~J. Zeevi, and Omar Besbes.
\newblock Stochastic multi-armed-bandit problem with non-stationary rewards.
\newblock In \emph{Advances in Neural Information Processing Systems 27: Annual
  Conference on Neural Information Processing Systems 2014, December 8-13 2014,
  Montreal, Quebec, Canada}, pages 199--207, 2014.

\bibitem[Gy{\"o}rgy et~al.(2007)Gy{\"o}rgy, Linder, Lugosi, and
  Ottucs{\'a}k]{Gyorgy2007}
A.~Gy{\"o}rgy, T.~Linder, G.~Lugosi, and G.~Ottucs{\'a}k.
\newblock The on-line shortest path problem under partial monitoring.
\newblock \emph{The Journal of Machine Learning Research}, 8:\penalty0
  2369--2403, 2007.

\bibitem[Kim and Tewari(2019)]{kim2019near}
Baekjin Kim and Ambuj Tewari.
\newblock Near-optimal oracle-efficient algorithms for stationary and
  non-stationary stochastic linear bandits.
\newblock \emph{arXiv preprint arXiv:1912.05695}, 2019.

\bibitem[Kveton et~al.(2014)Kveton, Wen, Ashkan, Eydgahi, and
  Eriksson]{kveton2014matroid}
Branislav Kveton, Zheng Wen, Azin Ashkan, Hoda Eydgahi, and Brian Eriksson.
\newblock Matroid bandits: Fast combinatorial optimization with learning.
\newblock \emph{arXiv preprint arXiv:1403.5045}, 2014.

\bibitem[Kveton et~al.(2015{\natexlab{a}})Kveton, Szepesv{\'a}ri, Wen, and
  Ashkan]{kveton2015cascading}
Branislav Kveton, Csaba Szepesv{\'a}ri, Zheng Wen, and Azin Ashkan.
\newblock Cascading bandits: learning to rank in the cascade model.
\newblock In \emph{Proceedings of the 32th International Conference on Machine
  Learning}, 2015{\natexlab{a}}.

\bibitem[Kveton et~al.(2015{\natexlab{b}})Kveton, Wen, Ashkan, and
  Szepesvari]{kveton2015combinatorial}
Branislav Kveton, Zheng Wen, Azin Ashkan, and Csaba Szepesvari.
\newblock Combinatorial cascading bandits.
\newblock \emph{Advances in Neural Information Processing Systems},
  2015{\natexlab{b}}.

\bibitem[Kveton et~al.(2015{\natexlab{c}})Kveton, Wen, Ashkan, and
  Szepesvari]{kveton2015tight}
Branislav Kveton, Zheng Wen, Azin Ashkan, and Csaba Szepesvari.
\newblock Tight regret bounds for stochastic combinatorial semi-bandits.
\newblock In \emph{Artificial Intelligence and Statistics}, pages 535--543,
  2015{\natexlab{c}}.

\bibitem[Lattimore and Szepesv{\'a}ri(2018)]{lattimore2018bandit}
Tor Lattimore and Csaba Szepesv{\'a}ri.
\newblock Bandit algorithms.
\newblock \emph{preprint}, page~28, 2018.

\bibitem[Liu et~al.(2018)Liu, Lee, and Shroff]{liu2018change}
Fang Liu, Joohyun Lee, and Ness~B. Shroff.
\newblock A change-detection based framework for piecewise-stationary
  multi-armed bandit problem.
\newblock In \emph{Proceedings of the Thirty-Second {AAAI} Conference on
  Artificial Intelligence, (AAAI-18), the 30th innovative Applications of
  Artificial Intelligence (IAAI-18), and the 8th {AAAI} Symposium on
  Educational Advances in Artificial Intelligence (EAAI-18), New Orleans,
  Louisiana, USA, February 2-7, 2018}, pages 3651--3658, 2018.

\bibitem[Luo et~al.(2018)Luo, Wei, Agarwal, and Langford]{luo2018efficient}
Haipeng Luo, Chen{-}Yu Wei, Alekh Agarwal, and John Langford.
\newblock Efficient contextual bandits in non-stationary worlds.
\newblock In \emph{Conference On Learning Theory, {COLT} 2018, Stockholm,
  Sweden, 6-9 July 2018.}, pages 1739--1776, 2018.

\bibitem[Robbins(1952)]{robbins1952bulletin}
Herbert Robbins.
\newblock Some aspects of the sequential design of experiments.
\newblock \emph{Bull. Amer. Math. Soc.}, 58\penalty0 (5):\penalty0 527--535, 09
  1952.

\bibitem[Russac et~al.(2019)Russac, Vernade, and Capp{\'e}]{russac2019weighted}
Yoan Russac, Claire Vernade, and Olivier Capp{\'e}.
\newblock Weighted linear bandits for non-stationary environments.
\newblock In \emph{Advances in Neural Information Processing Systems}, pages
  12017--12026, 2019.

\bibitem[Sherali(1987)]{sherali1987constructive}
Hanif~D Sherali.
\newblock A constructive proof of the representation theorem for polyhedral
  sets based on fundamental definitions.
\newblock \emph{American Journal of Mathematical and Management Sciences},
  7\penalty0 (3-4):\penalty0 253--270, 1987.

\bibitem[Thompson(1933)]{thompson1933likelihood}
William~R Thompson.
\newblock On the likelihood that one unknown probability exceeds another in
  view of the evidence of two samples.
\newblock \emph{Biometrika}, 25\penalty0 (3/4):\penalty0 285--294, 1933.

\bibitem[Wang et~al.(2019)Wang, Zhou, Li, Varshney, and Zhao]{wang2019aware}
Lingda Wang, Huozhi Zhou, Bingcong Li, Lav~R Varshney, and Zhizhen Zhao.
\newblock Be aware of non-stationarity: Nearly optimal algorithms for
  piecewise-stationary cascading bandits.
\newblock \emph{arXiv preprint arXiv:1909.05886}, 2019.

\bibitem[Wang and Chen(2017)]{wang2017improving}
Qinshi Wang and Wei Chen.
\newblock Improving regret bounds for combinatorial semi-bandits with
  probabilistically triggered arms and its applications.
\newblock In \emph{Advances in Neural Information Processing Systems}, pages
  1161--1171, 2017.

\bibitem[Wei et~al.(2016)Wei, Hong, and Lu]{wei2016tracking}
Chen{-}Yu Wei, Yi{-}Te Hong, and Chi{-}Jen Lu.
\newblock Tracking the best expert in non-stationary stochastic environments.
\newblock In \emph{Advances in Neural Information Processing Systems 29: Annual
  Conference on Neural Information Processing Systems 2016, December 5-10,
  2016, Barcelona, Spain}, pages 3972--3980, 2016.

\bibitem[Zhao and Chen(2019)]{zhao2019online}
Haoyu Zhao and Wei Chen.
\newblock Online second price auction with semi-bandit feedback under the
  non-stationary setting.
\newblock \emph{arXiv preprint arXiv:1911.05949}, 2019.

\bibitem[Zhou et~al.(2019)Zhou, Wang, Varshney, and Lim]{zhou2019near}
Huozhi Zhou, Lingda Wang, Lav~R Varshney, and Ee-Peng Lim.
\newblock A near-optimal change-detection based algorithm for
  piecewise-stationary combinatorial semi-bandits.
\newblock \emph{arXiv preprint arXiv:1908.10402}, 2019.

\bibitem[Zimmert et~al.(2019)Zimmert, Luo, and Wei]{zimmert2019beating}
Julian Zimmert, Haipeng Luo, and Chen-Yu Wei.
\newblock Beating stochastic and adversarial semi-bandits optimally and
  simultaneously.
\newblock In \emph{International Conference on Machine Learning}, pages
  7683--7692, 2019.

\end{thebibliography}



\clearpage
\newpage
\onecolumn

\section*{Appendix}
\section{Omitted Proofs in Section 3}
In this section, we give the performance guarantees of our algorithm $\Cucbsw$ and $\Cucbbob$ in the general case. We first give some definitions and prove some basic lemmas in the first part. Then, as a warm up, we prove the corresponding result of Theorem \ref{thm:cucb-sw-prob} in main content without the probabilistically triggered arms (Theorem \ref{thm:cucb-no-prob-trig} in appendix). Next, we prove Theorem \ref{thm:cucb-sw-prob} in main content with probabilistically triggered arms (Theorem \ref{thm:cucb-sw-prob-res} in appendix). Finally, we prove Theorem \ref{thm:cucb-bob} in main content (Theorem \ref{thm:cucb-bob-res} in appendix), which applies the Bandit-over-Bandit technique to achieve parameter-free.

\subsection{Fundamental definitions and tools}
First, we define the event-filtered regret. Generally speaking, it is the regret when some event happens.

\begin{definition}[Event-Filtered Regret]
    For any series of events $\{\cE_t\}_{t\ge 1}$ indexed by round number $t$, we define $\text{Reg}^{\cA}_{\alpha}(T,\{\cE_t\}_{t\ge 1})$ as the regret filtered by events $\{\cE_t\}_{t\ge 1}$, that is, regret is only counted in round $t$ if $\cE_t$ happens in round $t$. Formally,
    \[\text{Reg}^{\cA}_{\alpha}(T,\{\cE_t\}_{t\ge 1}) = \E\left[\sum_{t=1}^T\I\{\cE_t\}(\alpha\cdot\text{opt}_{\bmu_t} - r_{\bmu_t}(S_t^{\cA})\right].\]
    For convenience, $\cA$, $\alpha$, or $T$ can be omitted when the context is clear, and we simply use $\text{Reg}^{\cA}_{\alpha}(T,\cE_t)$ instead of $\text{Reg}^{\cA}_{\alpha}(T,\{\cE_t\}_{t\ge 1})$.
\end{definition}

Then, we define two important events that will use in the event-filtered regret. The two events are Sampling is Nice (Definition \ref{defn:sample-nice} and Triggering is Nice (Definition \ref{defn:triggering-nice}. We will also show that these two events happen with high probability. The following propositions, definitions, and lemmas are all related with these two definitions.

\begin{proposition}[Hoeffding Inequality]
    Suppose $X_i\in [0,1]$ for all $i\in [n]$ and $X_i$ are independent, then we have
    \[\Pr\left\{\bigg|\frac{1}{n}\sum_{i=1}^n X_i - \E\left[\frac{1}{n}\sum_{i=1}^n X_i\right]\bigg| \ge \varepsilon\right\} \le 2\exp\left(-2n\varepsilon^2\right).\]
\end{proposition}

\begin{definition}[Sampling is Nice]\label{defn:sample-nice}
    We say that the sampling is nice at the beginning of round $t$ if for any arm $i\in[m]$, we have $|\hat\mu_{i,t}-\nu_{i,t}| < \rho_{i,t}$, where $\rho_{i,t} = \sqrt{\frac{3\ln T}{2T_{i,t}}}$($\infty$ if $T_{i,t} = 0$) and $\hat\mu_{i,t}$ are defined in the algorithm, and
    \[
    \nu_{i,t} = \frac{1}{T_{i,t}}\sum_{s=t-w+1}^{t-1}\I\left\{i\text{ is triggered at time }s\right\} \mu_{i,t}.
    \]
    If $i$ is not triggered during time $(t-w,t-1]$, we define $\nu_{i,t} = \mu_{i,t}$. We use $\cN_{t}^s$ to denote this event.
\end{definition}

We have the following lemma saying that $\cN^s_t$ is a high probability event.

\begin{lemma}
    For each round $t\ge 1$, $\Pr\{\lnot \cN_{t}^s\} \le 2mT^{-2}$.
\end{lemma}

\begin{proof}
    The proof is a direct application of Hoeffding inequality and a union bound. First when $T_{i,t} = 0$, we have $\rho_{i,t} = \infty$ and the event $\cN_t^s$ happens. We first have
    \begin{align*}
        \Pr\{\lnot \cN_{t}^s\} =& \Pr\{\exists i\in[m], |\hat\mu_{i,t}-\nu_{i,t}| \ge \rho_{i,t}\}\\
            \le&\sum_{i=1}^m\Pr\{|\hat\mu_{i,t}-\nu_{i,t}| \ge \rho_{i,t}\}\\
            =&\sum_{i=1}^m\Pr\left\{|\hat\mu_{i,t}-\nu_{i,t}| \ge \sqrt{\frac{3\ln T}{2T_{i,t}}}\right\}\\
            =&\sum_{i=1}^m\sum_{k=1}^{\Gamma_t}\Pr\left\{T_{i,t} = k,|\hat\mu_{i,t}-\nu_{i,t}| \ge \sqrt{\frac{3\ln T}{2T_{i,t}}}\right\}.
    \end{align*}
    Then, by the conditional probability and the Hoeffding inequality, we have
    \begin{align*}
        &\Pr\left\{T_{i,t} = k,|\hat\mu_{i,t}-\nu_{i,t}| \ge \sqrt{\frac{3\ln T}{2T_{i,t}}}\right\}\\
        =&\Pr\{T_{i,t} = k\}\Pr\left\{|\hat\mu_{i,t}-\nu_{i,t}| \ge \sqrt{\frac{3\ln T}{2T_{i,t}}}\bigg|T_{i,t} = k\right\}\\
        \le&\Pr\{T_{i,t} = k\}2\exp\left(-2k\frac{3\ln T}{2k}\right)\\
        \le&2\exp\left(-2k\frac{3\ln T}{2k}\right)\\
        =& \frac{2}{T^3}.
    \end{align*}
    Then we know that
    \begin{align*}
         \Pr\{\lnot \cN_{t}^s\} \le&\sum_{i=1}^m\sum_{k=1}^{\Gamma_t}\Pr\left\{T_{i,t} = k,|\hat\mu_{i,t}-\nu_{i,t}| \ge \sqrt{\frac{3\ln T}{2T_{i,t}}}\right\}\\
         \le&\sum_{i=1}^m\sum_{k=1}^{\Gamma_t}\frac{2}{T^3}\\
         \le &\sum_{i=1}^m\sum_{k=1}^{t}\frac{2}{T^3}\\
         =& 2mT^{-2}.
    \end{align*}
\end{proof}

\begin{proposition}[Multiplicative Chernoff Bound]
    Suppose $X_i$ are Bernoulli variables for all $i\in [n]$ and $\E[X_i|X_1,\dots,X_{i-1}] \ge \mu$ for every $i\le n$. Let $Y = X_1 + \dots + X_n$, then we have
    \[\Pr\left\{Y \le(1-\delta)n\mu\right\} \le \exp\left(-\frac{\delta^2n\mu}{2}\right).\]
\end{proposition}

\begin{definition}[Triggering Probability (TP) Group]
    Let $i$ be an arm and $j$ be a positive natural number, define the triggering probability group (of actions)
    \[G_{i,j} = \{S^{{D}}\in\mathbb S\times \mathbb D| 2^{-j} < p_i^{{D},S}\le 2^{-j+1}\}.\]
\end{definition}

\begin{definition}[Main content definition \ref{defn:counter} restated]\label{defn:counter-res}
	Given the sliding window size $w$ of the algorithm, in a run of the algorithm, we define the counter $N_{i,j,t}$ as the following number
	\[N_{i,j,t} := \sum_{s=\max\{t-w+1,0\}}^t \I\left\{2^{-j} < p^{D_s,S_s}_i\le 2^{-j+1}\right\}.\]
\end{definition}

\begin{definition}[Triggering is Nice]\label{defn:triggering-nice}
    Given integers $\{j_{\max}^i\}_{i\in[m]}$, we call that the triggering is nice at the beginning of round $t$ if for any arm $i$ and any $1\le j\le j^i_{\max}$, as long as $6\ln t \le \frac{1}{3}N_{i,j,t-1}\cdot 2^{-j}$, we have
    \[T_{i,t-1} \ge \frac{1}{3}N_{i,j,t-1}\cdot 2^{-j}.\]
    We use $\cN^{t}_t$ to denote this event. 
\end{definition}

\begin{lemma}
    Given a series of integers $\{j_{\max}^i\}_{i\in [m]}$, we have for every round $t \ge 1$,
    \[\Pr\{\lnot\cN_t^t\}\le \sum_{i\in [m]}j_{\max}^i t^{-2}.\]
\end{lemma}

This lemma is exactly the same as Lemma 4 in \cite{wang2017improving}. The proof is a direct application of the Multiplicative Chernoff Bound. We omit the proof here.

Finally, we extend the definition of gap for the ease of the analysis. First recall that we have the following definition of gap.

\begin{definition}[Main content definition \ref{defn:gap} restated]
	For any distribution ${D}$ with mean vector $\bmu$. For each action $S$, we define the gap $\Delta^{{D}}_S := \max\{0,\alpha\cdot\text{opt}_{\bmu} - r_S(\bmu)\}$. For each arm $i$, we define
	\begin{align*}
	& \Delta^{i,t}_{\min} = \inf_{S\in\mathbb S:p^{{D}_t,S}_i > 0,\Delta^{{D}_t}_S > 0}\Delta^{{D}_t}_S, \\
	& \Delta^{i,t}_{\max} = \sup_{S\in\mathbb S:p^{{D}_t,S}_i > 0,\Delta^{{D}_t}_S > 0}\Delta^{{D}_t}_S.
	\end{align*}
	We define $\Delta_{\min}^i = +\infty$ and $\Delta_{\max}^i = 0$ if they are not properly defined by the above definitions. Furthermore, we define $\Delta^{i}_{\min} := \min_{t\le T}\Delta^{i,t}_{\min}$, $\Delta^{i}_{\max} := \max_{t\le T}\Delta^{i,t}_{\max}$ as the minimum and maximum gap for each arm.
\end{definition}

The previous definition of gap focus on a single distribution and a single arms. Furthermore, we define $\Delta^{t}_{\min} := \inf_{i\in [m]}\Delta^{i,t}_{\min}$, $\Delta^{t}_{\max} := \sup_{i\in [m]}\Delta^{i,t}_{\max}$ as the minimum and maximum gap in each round,  and $\Delta_{\min} := \inf_{t\le T}\Delta^{t}_{\min},\Delta_{\max} := \sup_{t\le T}\Delta^{t}_{\max}$ as the minimum and maximum gap.


\subsection{Non-stationary CMAB without probabilistically triggered arms}

As a warm up, we first consider the case without the probabilistically triggered arms, i.e. $p_i^{{D},S} \in\{0,1\}$. Then $\tilde S^{{D}} = S$ and we denote $K = \max_{S}|S|$. Then, the TPM bounded smoothness becomes the following,


\begin{restatable}[$1$-Norm Bounded Smoothness]{assumption}{assboundedsmooth}\label{ass:bounded-smoothness}
    For any two distributions ${D},{D}'$ with expectation vectors $\boldsymbol{\mu}$ and $\bmu'$ and any action $S$, we have
    \[|r_S(\bmu) - r_S(\bmu') |\le B\sum_{i\in S}|\bmu_i - \bmu'_i|.\]
\end{restatable}

We define the following number:
\[\kappa_T(M,s) = \left\{\begin{aligned}
    &2B\sqrt{6\ln T}, &\text{if } s = 0,\\
    &2B\sqrt{\frac{6\ln T}{s}}, &\text{if } 1\le s \le \ell_T(M),\\
    &0, &\text{if } s \ge \ell_T(M)+1,
\end{aligned}\right.\]
where
\[\ell_T(M)=\left\lfloor\frac{24B^2K^2\ln T}{M^2}\right\rfloor.\]
Generally speaking, we bridge the regret and the upper bound by this number, and we use the technique similar to that in \cite{wang2017improving}.

\begin{lemma}
    Suppose that the sliding window size is $w$. For any arm $i\in[m]$, any $T$, and any numbers $\{M_i\}_{i\le m}$,
    \[\sum_{t=1}^T \I(i\in S_t)\cdot \kappa_T(M_i,T_{i,t}) \le \left(\frac{T}{w}+1\right)\left(2B\sqrt{6\ln T} + \frac{48B^2K\ln T}{M_i}\right).\]
\end{lemma}

\begin{proof}
    We devide the time $\{1,2,\dots,T\}$ into the following $\Gamma$ segments $[1=t_0+1,w = t_1],[w+1 = t_1+1,2w = t_2],\dots,[t_{\Gamma-1}+1,t_{\Gamma} = T]$, where $t_{j-1} = t_j - w$. Each segment has length $w$, except for the last segment. It is easy to show that $\Gamma \le \left\lceil \frac{T}{w}\right\rceil$.
    
    Then we bound $\sum_{t=1}^T \I(i\in S_t)\cdot \kappa_T(M_i,T_{i,t})$. We first define another variable $T'_{i,t}$ for every $i,t$. Suppose that $t_{j-1} < t \le t_j$, which means that $t$ lies in the $j$th time segment, let $T'_{i,t}$ denote the number of times arm $i$ has been triggered in time $[t_{j-1}+1,t-1]$. 
    
    Then we know that $T_{i,t} \ge T'_{i,t}$, since the counter $T'_{i,t}$ counts the triggered times in a time interval which is a subset of the time interval for $T_{i,t}$. Because $\kappa_T(M,s)$ is decreasing when $s$ is increasing, we know that
    \[\sum_{t=1}^T \I(i\in S_t)\cdot \kappa_T(M_i,T_{i,t}) \le \sum_{t=1}^T \I(i\in S_t)\cdot \kappa_T(M_i,T'_{i,t})\]
    Then we bound the right hand side, and we have
    \begin{align*}
        \sum_{t=1}^T \I(i\in S_t)\cdot \kappa_T(M_i,T'_{i,t}) =& \sum_{j=1}^\Gamma\sum_{t=t_{j-1}+1}^{t_j} \I(i\in S_t)\cdot \kappa_T(M_i,T'_{i,t}) \\
        \le& \sum_{j=1}^\Gamma\sum_{s=0}^{w-1}\kappa_T(M_i,s) \\
        \le& \sum_{j=1}^\Gamma\left(2B\sqrt{6\ln T} + \sum_{s=1}^{\ell_T(M_i)}\kappa_T(M_i,s)\right)\\
        =& \sum_{j=1}^\Gamma\left(2B\sqrt{6\ln T} + \sum_{s=1}^{\ell_T(M_i)}2B\sqrt{\frac{6\ln T}{s}}\right)\\
        \le& \sum_{j=1}^\Gamma\left(2B\sqrt{6\ln T} + \int_{0}^{\ell_T(M_i)}2B\sqrt{\frac{6\ln T}{s}}ds\right)\\
        \le& \sum_{j=1}^\Gamma\left(2B\sqrt{6\ln T} + 4B\sqrt{6\ln T\ell_T(M_i)}\right)\\
        \le& \sum_{j=1}^\Gamma\left(2B\sqrt{6\ln T} + 4B\sqrt{6\ln T \frac{24B^2K^2\ln T}{M_i^2}}\right)\\
        \le&\left(\frac{T}{w} + 1\right)\left(2B\sqrt{6\ln T} + \frac{48B^2K\ln T}{M_i}\right).
    \end{align*}
\end{proof}

Then, we have the following simple lemma to bound the difference between the true mean of each round and the actual mean for the round that we trigger. The lemma is simple to proof, and a detailed proof can be found in \cite{zhao2019online}.

\begin{lemma}
    Suppose that the size of the sliding window is $w$. For every $t$ and every possible triggering, we have 
    \[||\nu_{t} - \mu_t||_{\infty} \le \sum_{s=t-w+2}^{t}||\mu_s - \mu_{s-1}||_{\infty}.\]
\end{lemma}


Denote $\Delta_{S}^t$ as $\Delta_{S}^{\mc{D}_t}$ for simplicity. At round $t$ with action $S_t$, we use $\Delta_{S_t}$ for short.

\begin{lemma}\label{lem:cucb-sw-key-no-prob}
    Suppose that the size of the sliding window is $w$ and fix the parameters $M_i$ for each $i\in [m]$ and defining $M_{S_t} = \max_{i\in S_t}M_i$. Then we have
    \[\text{Reg}(\{\Delta^t_{S_t} \ge M_{S_t}\}\land\cN_t^s\land\lnot \cF_t) \le \sum_{i\in[m]}\left(\frac{T}{w} + 1\right)\left(2B\sqrt{6\ln T} + \frac{48B^2K\ln T}{M_i}\right)+2(1+\alpha)KB\sum_{s=2}^{t}||\mu_s - \mu_{s-1}||_{\infty}\cdot w.\]
    where $\cF_t$ is denoted as the event that $\{r_{S_t}(\bar \mu_t) < \alpha\cdot\text{opt}_{\bar\mu_t}\}$
\end{lemma}

\begin{proof}
    From the assumption of our oracle, we know that $\Pr\{\cF_t\} \le 1-\beta$. We also define $M_{S} = \max_{i\in\bar S}M_i$ for each possible action $S$, and use define $M_S = 0$ if $\bar S = \phi$. We first show that when $\{\Delta^t_{S_t} \ge M_{S_t}\},\cN_t^s,\lnot \cF_t$ all happens, we have
    \[\Delta^t_{S_t} \le \sum_{i\in\bar S_t}\kappa_T(M_i,T_{i,t-1})+2(1+\alpha)KB\sum_{s=t-w+2}^{t}||\mu_s - \mu_{s-1}||_{\infty}.\]
    
    First when $\Delta^t_{S_t} = 0$, the inequality holds, and we just have to prove the case when $\Delta^t_{S_t} > 0$. Let $R_1$ denote the optimal strategy when the mean vector is $\mu'_t$ in which the $i$-th entry is $\mu'_{i,t} = \min\{\nu_{i,t}+\sum_{s=t-w+2}^{t}||\mu_s - \mu_{s-1}||_{\infty},1\}$. Then we know that $\mu'_{i,t}\ge\mu_{i,t}$. From $\cN_{t}^s$ and $\lnot \cF_t$, we have
    \begin{align*}
    r_{S_t}(\bar\mu_t) \ge& \alpha\cdot \text{opt}_{\bar\mu_t} \ge \alpha\cdot r_{R_1}(\bar\mu_t)\ge \alpha\cdot r_{R_1}(\nu_t)\\
    \ge& \alpha\cdot r_{R_1}(\mu'_t)-\alpha KB\sum_{s=t-w+2}^{t}||\mu_s - \mu_{s-1}||_{\infty}\\
    \ge& \alpha\cdot \text{opt}_{\mu_t}-\alpha KB\sum_{s=t-w+2}^{t}||\mu_s - \mu_{s-1}||_{\infty}\\
    =& r_{S_t}(\mu_t)+\Delta^t_{S_t}-\alpha KB\sum_{s=t-w+2}^{t}||\mu_s - \mu_{s-1}||_{\infty}\\
    \ge& r_{S_t}(\nu_t)+\Delta^t_{S_t}-(1+\alpha) KB\sum_{s=t-w+2}^{t}||\mu_s - \mu_{s-1}||_{\infty},
    \end{align*}
    so we get
    \begin{align*}
        \Delta_{S_t} \le& r_{S_t}(\bar \mu_t) - r_{S_t}(\nu_t) + (1+\alpha) KB\sum_{s=t-w+2}^{t}||\mu_s - \mu_{s-1}||_{\infty}\\
        \le& B\sum_{i\in S_t}(\bar\mu_{i,t}-\nu_{i,t}) + (1+\alpha) KB\sum_{s=t-w+2}^{t}||\mu_s - \mu_{s-1}||_{\infty}.
    \end{align*}
    Then when $\{\Delta^t_{S_t} \ge M_{S_t}\},\cN_t^s,\lnot \cF_t$ all happens, we have
    \begin{align*}
        \Delta^t_{S_t} \le& B\sum_{i\in S_t}(\bar\mu_{i,t}-\nu_{i,t}) + (1+\alpha) KB\sum_{s=t-w+2}^{t}||\mu_s - \mu_{s-1}||_{\infty}\\
        \le& -M_{S_t}+ 2B\sum_{i\in  S_t}(\bar\mu_{i,t}-\nu_{i,t}) + 2(1+\alpha) KB\sum_{s=t-w+2}^{t}||\mu_s - \mu_{s-1}||_{\infty}\\
        \le& 2B\sum_{i\in  S_t}\left(\bar\mu_{i,t}-\nu_{i,t}-\frac{M_{S_t}}{2B|\bar S_t|}\right) + 2(1+\alpha) KB\sum_{s=t-w+2}^{t}||\mu_s - \mu_{s-1}||_{\infty}\\
        \le& 2B\sum_{i\in  S_t}\left(\bar\mu_{i,t}-\nu_{i,t}-\frac{M_{S_t}}{2BK}\right) + 2(1+\alpha) KB\sum_{s=t-w+2}^{t}||\mu_s - \mu_{s-1}||_{\infty}\\
        \le& 2B\sum_{i\in  S_t}\left(\bar\mu_{i,t}-\nu_{i,t}-\frac{M_{i}}{2BK}\right) + 2(1+\alpha) KB\sum_{s=t-w+2}^{t}||\mu_s - \mu_{s-1}||_{\infty}.
    \end{align*}
    By the same proof in \cite{wang2017improving}, it can be shown that
    \[2B\sum_{i\in S_t}\left(\bar\mu_{i,t}-\nu_{i,t}-\frac{M_{i}}{2BK}\right) \le \sum_{i\in  S_t}\kappa_T(M_i,T_{i,t-1}),\]
    and thus we have
    \[\Delta^t_{S_t} \le \sum_{i\in S_t}\kappa_T(M_i,T_{i,t-1})+2(1+\alpha)KB\sum_{s=t-w+2}^{t}||\mu_s - \mu_{s-1}||_{\infty}.\]
    From the previous 2 lemmas, we know that
    \[\text{Reg}(\{\Delta^t_{S_t} \ge M_{S_t}\}\land\cN_t^s\land\lnot \cF_t) \le \sum_{i\in[m]}\left(\frac{T}{w} + 1\right)\left(2B\sqrt{6\ln T} + \frac{48B^2K\ln T}{M_i}\right)+2(1+\alpha)KB\sum_{s=2}^{t}||\mu_s - \mu_{s-1}||_{\infty}\cdot w.\]
\end{proof}

\begin{theorem}\label{thm:cucb-no-prob-trig}
    Choosing the length of the sliding window to be $w = \min\left\{\sqrt{\frac{T}{{V}}},T\right\}$, we have the following distribution dependent bound,
\[\text{Reg}_{\alpha,\beta} = \tilde O\left(\sum_{i\in [m]}\frac{K\sqrt{{V} T}}{\Delta^i_{\min}} + \sum_{i\in [m]}\frac{K}{\Delta^i_{\min}} + mK \right).\]
If we choose the length of the sliding window to be $w = \min\left\{m^{1/3}T^{2/3}K^{-1/3}{V}^{-2/3},T\right\}$, we have the following distribution independent bound,
\[\text{Reg}_{\alpha,\beta} = \tilde O\left((m{V})^{1/3}(KT)^{2/3} + \sqrt{mKT} + mK\right).\]
\end{theorem}

The proof is the same as the proof of Theorem \ref{thm:cucb-sw-prob-res}, and we omit the proof here. The only difference is that, without the probabilistically triggered arms, the constants in Lemma \ref{lem:cucb-sw-key-no-prob} is better than the corresponding lemma with the probabilistically triggered arms.

\subsection{Non-stationary CMAB with probabilistically triggered arms}
In this part, we consider the case with probabilistically triggered arms. Recall that the we have the main TPM bounded smoothness assumption,

\begin{assumption}[Main content assumption \ref{ass:tpm-bounded-smoothness} restated]
	For any two distributions ${D},{D}'$ with expectation vectors $\boldsymbol{\mu}$ and $\bmu'$ and any action $S$, we have
	\[|r_S(\bmu) - r_S(\bmu') |\le B\sum_{i\in [m]}p^{{D},S}_i|\bmu_i - \bmu'_i|.\]
\end{assumption}

Recall that $\tilde S^{{D}} = \{i\in[m]:p_i^{{D},S} > 0\}$ is the set that can be triggered by action $S$ with distribution ${D}$, and we denote $K = \max_{S_{{D}}}|\tilde S|$. We define the following number:
\[\kappa_{j,T}(M,s) = \left\{\begin{aligned}
    &2B\sqrt{72\cdot 2^{-j}\cdot \ln T}, &\text{if } s = 0,\\
    &2B\sqrt{\frac{72\cdot 2^{-j}\cdot \ln T}{s}}, &\text{if } 1\le s \le \ell_{j,T}(M),\\
    &0, &\text{if } s \ge \ell_{j,T}(M)+1,
\end{aligned}\right.\]
where
\[\ell_{j,T}(M)=\left\lfloor\frac{288\cdot 2^{-j}\cdot B^2K^2\ln T}{M^2}\right\rfloor.\]

This number is similar to the number defined in the previous part, but this time, we need to consider the probabilistically triggered arms. Besides the $M,s$ that are taken as inputs, we also have $j$ and $T$ as parameters.

\begin{lemma}\label{lem:prob-trig-kappa-lemma}
    If $\{\Delta_{S_t} \ge M_{S_t}\},\lnot \cF_t,\cN_t^s$ and $\cN_t^t$ hold, we have
    \[\Delta_{S_t}\le\sum_{i\in\tilde S_t^{{D}_t}}\kappa_{j_i,T}(M_i,N_{i,j_i,t-1})+2(1+\alpha)KB\sum_{s=t-w+2}^{t}||\mu_s - \mu_{s-1}||_{\infty},\]
    where $j_i$ is the index of the TP group with $S_t^{{D}_t}\in G_{i,j_i}$.
\end{lemma}

\begin{proof}
    First, similar to the proof with no probabilistic triggering arms, we use the back amortization trick.
    
    First when $\Delta_{S_t} = 0$, the inequality holds, and we just have to prove the case when $\Delta_{S_t} > 0$. Let $R_1$ denote the optimal strategy when the mean vector is $\bmu'_t$, where $\bmu'_t$ is the vector constituted by $\mu'_{i,t} = \min\{\nu_{i,t}+\sum_{s=t-w+2}^{t}||\mu_s - \mu_{s-1}||_{\infty},1\}$. Then we know that $\mu'_{i,t}\ge\mu_{i,t}$. From $\cN_{t}^s$ and $\lnot \cF_t$, we have
    \begin{align*}
    r_{S_t}(\bar\mu_t) \ge& \alpha\cdot \text{opt}_{\bar\mu_t} \ge \alpha\cdot r_{R_1}(\bar\mu_t)\ge \alpha\cdot r_{R_1}(\nu_t)\\
    \ge& \alpha\cdot r_{R_1}(\mu'_t)-\alpha KB\sum_{s=t-w+2}^{t}||\mu_s - \mu_{s-1}||_{\infty}\\
    \ge& \alpha\cdot \text{opt}_{\mu_t}-\alpha KB\sum_{s=t-w+2}^{t}||\mu_s - \mu_{s-1}||_{\infty}\\
    =& r_{S_t}(\mu_t)+\Delta_{S_t}-\alpha KB\sum_{s=t-w+2}^{t}||\mu_s - \mu_{s-1}||_{\infty}\\
    \ge& r_{S_t}(\nu_t)+\Delta_{S_t}-(1+\alpha) KB\sum_{s=t-w+2}^{t}||\mu_s - \mu_{s-1}||_{\infty},
    \end{align*}
    so we get
    \begin{align*}
        \Delta_{S_t} \le& r_{S_t}(\bar \mu_t) - r_{S_t}(\nu_t) + (1+\alpha) KB\sum_{s=t-w+2}^{t}||\mu_s - \mu_{s-1}||_{\infty}\\
        \le& B\sum_{i\in\tilde S_t}p_i^{{D}_t,S_t}(\bar\mu_{i,t}-\nu_{i,t}) + (1+\alpha) KB\sum_{s=t-w+2}^{t}||\mu_s - \mu_{s-1}||_{\infty}.
    \end{align*}
    Then when $\{\Delta^t_{S_t} \ge M_{S_t}\},\cN_t^s,\lnot \cF_t$ all happens, we have
    \begin{align*}
        \Delta_{S_t} \le& B\sum_{i\in\tilde S_t}p_i^{{D}_t,S_t}(\bar\mu_{i,t}-\nu_{i,t}) + (1+\alpha) KB\sum_{s=t-w+2}^{t}||\mu_s - \mu_{s-1}||_{\infty}\\
        \le& -M_{S_t}+ 2B\sum_{i\in\tilde S_t}p_i^{{D}_t,S_t}(\bar\mu_{i,t}-\nu_{i,t}) + 2(1+\alpha) KB\sum_{s=t-w+2}^{t}||\mu_s - \mu_{s-1}||_{\infty}\\
        \le& 2B\sum_{i\in\tilde S_t}p_i^{{D}_t,S_t}\left(\bar\mu_{i,t}-\nu_{i,t}-\frac{M_{S_t}}{2B|\tilde S_t|}\right) + 2(1+\alpha) KB\sum_{s=t-w+2}^{t}||\mu_s - \mu_{s-1}||_{\infty}\\
        \le& 2B\sum_{i\in\tilde S_t}p_i^{{D}_t,S_t}\left(\bar\mu_{i,t}-\nu_{i,t}-\frac{M_{S_t}}{2BK}\right) + 2(1+\alpha) KB\sum_{s=t-w+2}^{t}||\mu_s - \mu_{s-1}||_{\infty}\\
        \le& 2B\sum_{i\in\tilde S_t}p_i^{{D}_t,S_t}\left(\bar\mu_{i,t}-\nu_{i,t}-\frac{M_{i}}{2BK}\right) + 2(1+\alpha) KB\sum_{s=t-w+2}^{t}||\mu_s - \mu_{s-1}||_{\infty}.
    \end{align*}
    Because of $\cN_t^t$, same as the proof of Lemma 5 of \cite{wang2017improving}, we can show that
    \[2B\sum_{i\in\tilde S_t}p_i^{{D}_t,S_t}\left(\bar\mu_{i,t}-\nu_{i,t}-\frac{M_{i}}{2BK}\right) \le \sum_{i\in(\tilde S_t)^{{D}_t}}\kappa_{j_i,T}(M_i,N_{i,j_i,t-1}).\]
    In this way, we prove the following inequality
    \[\Delta_{S_t}\le\sum_{i\in(\tilde S_t)^{{D}_t}}\kappa_{j_i,T}(M_i,N_{i,j_i,t-1})+2(1+\alpha)KB\sum_{s=t-w+2}^{t}||\mu_s - \mu_{s-1}||_{\infty},\]
    when $\{\Delta_{S_t} \ge M_{S_t}\},\lnot \cF_t,\cN_t^s$ and $\cN_t^t$ hold.
\end{proof}

Then we have the following main lemma to bound the regret with probabilistically triggered arms.

\begin{lemma}\label{lem:prob-trig-key}
    Suppose that the size of the sliding window is $w$ and fix choose the parameters $M_i$ for each $i\in [m]$ and defining $M_{S_t} = \min_{i\in\hat S}M_i$. Then we have
    \begin{align*}
        &\text{Reg}(\{\Delta^{{D}_t}_{S_t} \ge M_{S_t}\}\land\cN_t^s\land\cN_t^t\land\lnot \cF_t)\\
        \le& \sum_{i\in [m]}\left(\frac{T}{w}+1\right)\left(12(2+\sqrt{2})B\sqrt{\ln T} + \frac{576B^2K\ln T}{M_i}\right)+2(1+\alpha)KB\sum_{s=2}^{t}||\mu_s - \mu_{s-1}||_{\infty}\cdot w.
    \end{align*}
\end{lemma}

\begin{proof}
    From Lemma \ref{lem:prob-trig-kappa-lemma}, we know that when $\{\Delta^{{D}_t}_{S_t} \ge M_{S_t}\},\lnot \cF_t,\cN_t^s$ and $\cN_t^t$ hold, we have
    \[\Delta^{{D}_t}_{S_t}\le\sum_{i\in(\tilde S_t)^{{D}_t}}\kappa_{j_i,T}(M_i,N_{i,j_i,t-1})+2(1+\alpha)KB\sum_{s=t-w+2}^{t}||\mu_s - \mu_{s-1}||_{\infty}.\]
    Then, sum over $t=1,\dots,T$, we have
    \begin{align*}
        \text{Reg}(\{\Delta^{{D}_t}_{S_t} \ge M_{S_t}\}\land\cN_t^s\land\cN_t^t\land\lnot \cF_t) \le& \sum_{t=1}^{T}\sum_{i\in(\tilde S_t)^{{D}_t}}\kappa_{j_i,T}(M_i,N_{i,j_i,t-1})+2(1+\alpha)KB\sum_{t=1}^{T}\sum_{s=t-w+2}^{t}||\mu_s - \mu_{s-1}||_{\infty} \\
        \le& \sum_{t=1}^{T}\sum_{i\in(\tilde S_t)^{{D}_t}}\kappa_{j_i,T}(M_i,N_{i,j_i,t-1})+2(1+\alpha)KB\sum_{s=2}^{t}||\mu_s - \mu_{s-1}||_{\infty}\cdot w. 
    \end{align*}
    Then we bound the first term. Like the proof without probabilistically triggered arms, we construct another counter $N'_{i,j,t-1}$, which lower bound $N_{i,j,t-1}$. We divide the time $\{1,2,\dots,T\}$ into the following $\Gamma$ segments $[1=t_0+1,w = t_1],[w+1 = t_1+1,2w = t_2],\dots,[t_{\Gamma-1}+1,t_{\Gamma} = T]$, where $t_{j-1} = t_j - w$. Each segment has length $w$, except for the last segment. It is easy to show that $\Gamma \le \left\lceil \frac{T}{w}\right\rceil$. Suppose that $t_{k-1} < t \le t_k$, then define
    \[N'_{i,j,t} := \sum_{s=t_k+1}^t \I\left\{2^{-j} < p^{{D}_s,S_s}_i\le 2^{-j+1}\right\}.\]
    
    Because $\kappa_{j,T}(M,s)$ is monotonically decreasing in terms of $s$, we have
    \begin{align*}
        &\sum_{t=1}^{T}\sum_{i\in(\tilde S_t)^{{D}_t}}\kappa_{j_i,T}(M_i,N_{i,j_i,t-1})\\
        \le& \sum_{t=1}^{T}\sum_{i\in(\tilde S_t)^{{D}_t}}\kappa_{j_i,T}(M_i,N'_{i,j_i,t-1}) \\
        \le& \sum_{i\in [m]}\sum_{k=1}^{\Gamma}\sum_{j=1}^{+\infty}\sum_{s=t_{k-1}+1}^{t_k} \kappa_{j,T}(M_i,s-t_{k-1}-1) \\
        \le& \sum_{i\in [m]}\sum_{k=1}^{\Gamma}\sum_{j=1}^{+\infty}\sum_{s=0}^{\ell_{j,T}(M_i)} \kappa_{j,T}(M_i,s-t_{k-1}-1) \\
        \le& \sum_{i\in [m]}\sum_{k=1}^{\Gamma}\sum_{j=1}^{+\infty}\left(2B\sqrt{72\cdot 2^{-j}\cdot \ln T} + \sum_{s=1}^{\ell_{j,T}(M_i)} 2B\sqrt{\frac{72\cdot 2^{-j}\cdot \ln T}{s}}\right)\\
        \le& \sum_{i\in [m]}\sum_{k=1}^{\Gamma}\sum_{j=1}^{+\infty}\left(2B\sqrt{72\cdot 2^{-j}\cdot \ln T} + 2\cdot 2B\sqrt{72\cdot 2^{-j}\cdot \ln T}\cdot\sqrt{\ell_{j,T}(M_i)} \right) \\
        \le& \sum_{i\in [m]}\sum_{k=1}^{\Gamma}\sum_{j=1}^{+\infty}\left(2B\sqrt{72\cdot 2^{-j}\cdot \ln T} + 2\cdot 2B\sqrt{72\cdot 2^{-j}\cdot \ln T}\cdot\sqrt{\frac{288\cdot 2^{-j}\cdot B^2K^2\ln T}{M_i^2}} \right) \\
        \le& \sum_{i\in [m]}\sum_{k=1}^{\Gamma}\left(12(2+\sqrt{2})B\cdot\sqrt{\ln T} + \frac{576\cdot B^2K\cdot\ln T}{M_i}\right) \\
        \le& \sum_{i\in [m]}\left(\frac{T}{w}+1\right)\left(12(2+\sqrt{2})B\cdot\sqrt{\ln T} + \frac{576\cdot B^2K\cdot\ln T}{M_i}\right).
    \end{align*}
    Then combining with Lemma \ref{lem:prob-trig-kappa-lemma}, we have
    \begin{align*}
        &\text{Reg}(\{\Delta^{{D}_t}_{S_t} \ge M_{S_t}\}\land\cN_t^s\land\cN_t^t\land\lnot \cF_t)\\
        \le& \sum_{i\in [m]}\left(\frac{T}{w}+1\right)\left(12(2+\sqrt{2})B\sqrt{\ln T} + \frac{576B^2K\ln T}{M_i}\right)+2(1+\alpha)KB\sum_{s=2}^{t}||\mu_s - \mu_{s-1}||_{\infty}\cdot w.
    \end{align*}
\end{proof}

\begin{theorem}[Main content theorem \ref{thm:cucb-sw-prob} restated]\label{thm:cucb-sw-prob-res}
	Choosing the length of the sliding window to be $w = \min\left\{\sqrt{\frac{T}{{V}}},T\right\}$, we have the following distribution dependent bound,
	\[\text{Reg}_{\alpha,\beta} = \tilde{O}\left(\sum_{i\in [m]}\frac{K\sqrt{{V} T}}{\Delta^i_{\min}} + \sum_{i\in [m]}\frac{K}{\Delta^i_{\min}} + mK \right).\]
	If we choose the length of the sliding window to be $w = \min\left\{m^{1/3}T^{2/3}K^{-1/3}{V}^{-2/3},T\right\}$, we have the following distribution independent bound,
	\[\text{Reg}_{\alpha,\beta} = \tilde{O}\left((m{V})^{1/3}(KT)^{2/3} + \sqrt{mKT} + mK\right).\]
\end{theorem}

\begin{proof}
    First, from the definition of the filtered regret, we know that
    \begin{align*}
        \text{Reg}(\{\}) \le& \text{Reg}(\{\Delta^{{D}_t}_{S_t} \ge M_{S_t}\}\land\cN_t^s\land\cN_t^t\land\lnot \cF_t) + \text{Reg}(\{\Delta^{{D}_t}_{S_t} < M_{S_t}\}) + \text{Reg}(\lnot \cN_t^s) + \text{Reg}(\lnot \cN_t^t) + \text{Reg}(\cF_t).
    \end{align*}
    The last 3 terms are rather easy to bound, we have
    \begin{align*}
        \text{Reg}(\lnot \cN_t^s) =& \sum_{t=1}^T \Delta^{{D}_t}_{S_t}\I\{\lnot \cN_t^s\} \le \sum_{t=1}^T\Pr\{\lnot \cN_t^s\}\Delta_{\max} \le \frac{\pi^2}{3}m\cdot\Delta_{\max}\\
        \text{Reg}(\lnot \cN_t^t) =& \sum_{t=1}^T \Delta^{{D}_t}_{S_t}\I\{\lnot \cN_t^t\} \le \sum_{t=1}^T\Pr\{\lnot \cN_t^t\}\Delta_{\max} \le \frac{\pi^2}{6}\sum_{i\in [m]}j^i_{\max}\cdot\Delta_{\max}\\
        \text{Reg}(\cF_t) =& \sum_{t=1}^T \Delta^{{D}_t}_{S_t}\I\{\cF_t\} \le \sum_{t=1}^T\Pr\{\cF_t\}\Delta^t_{\max} \le (1-\beta)\cdot\sum_{t=1}^T\Delta^t_{\max}
    \end{align*}
    We also know that
    \begin{align*}
        &\text{Reg}^{\cA}_{\alpha,\beta}-\text{Reg}(\{\Delta^{{D}_t}_{S_t} < M_{S_t}\})\\
        =& \alpha\cdot\beta\cdot\sum_{t=1}^T\text{opt}_{\boldsymbol\mu_t} - \E\left[\sum_{t=1}^T r_{S_t^{\cA}}(\boldsymbol\mu_t)\right]-\text{Reg}(\{\Delta^{{D}_t}_{S_t} < M_{S_t}\})\\
        =& \text{Reg}(\{\}) - (1-\beta)\alpha\cdot \sum_{t=1}^T\text{opt}_{\boldsymbol\mu_t} -\text{Reg}(\{\Delta^{{D}_t}_{S_t} < M_{S_t}\})\\
        \le& \text{Reg}(\{\Delta^{{D}_t}_{S_t} \ge M_{S_t}\}\land\cN_t^s\land\cN_t^t\land\lnot \cF_t) + \text{Reg}(\lnot \cN_t^s) + \text{Reg}(\lnot \cN_t^t) + \text{Reg}(\cF_t) - (1-\beta)\alpha\cdot \sum_{t=1}^T\text{opt}_{\boldsymbol\mu_t}\\
        \le& \text{Reg}(\{\Delta^{{D}_t}_{S_t} \ge M_{S_t}\}\land\cN_t^s\land\cN_t^t\land\lnot \cF_t) +\frac{\pi^2}{3}m\cdot\Delta_{\max} + \frac{\pi^2}{6}\sum_{i\in [m]}j^i_{\max}\cdot\Delta_{\max} \\
        &\quad + (1-\beta)\cdot\sum_{t=1}^T\Delta^t_{\max} - (1-\beta)\alpha\cdot \sum_{t=1}^T\text{opt}_{\boldsymbol\mu_t} \\
        \le& \text{Reg}(\{\Delta^{{D}_t}_{S_t} \ge M_{S_t}\}\land\cN_t^s\land\cN_t^t\land\lnot \cF_t) +\frac{\pi^2}{3}m\cdot\Delta_{\max} + \frac{\pi^2}{6}\sum_{i\in [m]}j^i_{\max}\cdot\Delta_{\max}.
    \end{align*}
    Then we have
    \[\text{Reg}^{\cA}_{\alpha,\beta}\le \text{Reg}(\{\Delta^{{D}_t}_{S_t} \ge M_{S_t}\}\land\cN_t^s\land\cN_t^t\land\lnot \cF_t) +\text{Reg}(\{\Delta^{{D}_t}_{S_t} < M_{S_t}\})+\frac{\pi^2}{3}m\cdot\Delta_{\max} + \frac{\pi^2}{6}\sum_{i\in [m]}j^i_{\max}\cdot\Delta_{\max}.\]
    Recall that from Lemma \ref{lem:prob-trig-key},
    \begin{align*}
        &\text{Reg}(\{\Delta^{{D}_t}_{S_t} \ge M_{S_t}\}\land\cN_t^s\land\cN_t^t\land\lnot \cF_t)\\
        \le& \sum_{i\in [m]}\left(\frac{T}{w}+1\right)\left(12(2+\sqrt{2})B\sqrt{\ln T} + \frac{576B^2K\ln T}{M_i}\right)+2(1+\alpha)KB\sum_{s=2}^{t}||\mu_s - \mu_{s-1}||_{\infty}\cdot w.
    \end{align*}
    For the distribution dependent bound, we choose $M_i = \Delta^i_{\min}$. Then, we have $\Delta^{{D}_t}_{S_t} \ge M_{S_t}$ and $\text{Reg}(\{\Delta^{{D}_t}_{S_t} < M_{S_t}\}) = 0$. If we set $w = \min\left\{\sqrt{\frac{T}{{V}}},T\right\}$, we can get
    \[\text{Reg}^{\cA}_{\alpha,\beta} = \tilde O\left(\sum_{i\in [m]}\frac{K\sqrt{{V} T}}{\Delta^i_{\min}} + \sum_{i\in [m]}\frac{K}{\Delta^i_{\min}} + mK \right).
    \]
    As for the distribution independent bound, if we set $w = \min\left\{m^{1/3}T^{2/3}K^{-1/3}{V}^{-2/3},T\right\}, M_i = \sqrt{mK / w} = \Theta(\max\{(mK{V})^{1/3}T^{-1/3}),\sqrt{mK/T}\}$, we can get
    \[\text{Reg}^{\cA}_{\alpha,\beta} = \tilde O\left((m{V})^{1/3}(KT)^{2/3} + \sqrt{mKT} + mK\right) = \tilde O\left((m{N})^{1/3}(KT)^{2/3} + \sqrt{mKT} + mK\right).\]
\end{proof}

\subsection{Theoretical guarantees of $\Cucbbob$}
In this section, we show the performance guarantee of our algorithm $\Cucbbob$. Before moving into the formal proof, we will first introduce more on the EXP3 algorithm and its variant: EXP3.P algorithm.

\paragraph{Background on the EXP3 algorithm and its variant}

First we introduce the EXP3 algorithm and its variant EXP3.P algorithm. EXP3 algorithm is a famous algorithm for the adversarial bandit problem. In the original paper that introduce the Bandit-over-Bandit technique \cite{cheung2019learning}, the authors apply the EXP3 algorithm. However in our case, the regret is complicated and to make the proof easier, we apply the EXP3.P algorithm. The difference is that, the EXP3 algorithm has bounded ``pseudo-regret'', but the EXP3.P algorithm has bounded ``regret'' with high probability, and thus has bounded ``expected regret''. It is know that the ``pseudo-regret'' is a weaker measurement than the ``expected regret'', so for the ease of analysis, we apply EXP3.P algorithm.

\begin{algorithm}
    \caption{EXP3.P}
    \label{alg:exp3p}
    \begin{algorithmic}[1]
        \STATE {\bfseries Input: } Number of arms $K'$, Total time horizon $T'$, Parameters $\eta\in\R^+$, $\gamma,\beta\in [0,1]$.
        \STATE Let $p_1$ denote the uniform distribution over $[K']$.
        \FOR{$t=1,2,\dots,T'$}
            \STATE Draw an arm $I_t$ according to the probability distribution $p_t$.
            \STATE Compute the estimated gain for each arm
            \[\tilde g_{i,t} = \frac{g_{i,t}\I\{I_t = i\} + \beta}{p_{i,t}}\]
            \STATE Update the estimated gain $\tilde G_{i,t} = \sum_{s=1}^t \tilde g_{i,s}$.
            \STATE Compute the new probability distribution over the arms $p_{t+1} = (p_{1,t+1},\dots,p_{K',t+1})$, where
            \[p_{i,t+1} = (1-\gamma) \frac{\exp(\eta \tilde G_{i,t})}{\sum_{k=1}^{K'}\exp(\eta \tilde G_{k,t})} + \frac{\gamma}{K'}.\]
        \ENDFOR
    \end{algorithmic}
\end{algorithm}

Algorithm \ref{alg:exp3p} is the pseudo-code for the EXP3.P algorithm. In the algorithm, $p_{i,t}$ is the gain (reward) in round $t$ of arm $i$, and it satisfies $0\le p_{i,t}\le 1$. It is easy to generalize the algorithm into the case where $0\le p_{i,t} \le R'$, and we only have to normalize to $[0,1]$ each time.

By choosing the parameters
\[\beta = \sqrt{\frac{\ln K'}{K'T'}},\eta = 0.95\sqrt{\frac{\ln K'}{T'K'}},\gamma = 1.05\sqrt{\frac{K'\ln K'}{T'}},\]
we have the following performance guarantee for the EXP3.P algorithm.

\begin{proposition}[Main content proposition \ref{prop:exp3p-regret} restated]\label{prop:exp3p-regret-res}
	Suppose that the reward of each arm in each round is bounded by $0\le r_{i,t}\le R'$, the number of arms is $K'$, and the total time horizon is $T'$. The expected regret of EXP3.P algorithm is bounded by ${O}(R'\sqrt{K'T'\log K'})$.
\end{proposition}

\paragraph{Proof of Theorem \ref{thm:cucb-bob} in main content}

Now we prove Theorem \ref{thm:cucb-bob} in main content (Theorem \ref{thm:cucb-bob-res} in appendix). The main part of the proof is to decompose the regret into 2 parts, and optimize the length of each block to balance 2 parts. Recall that we have the following theorem.

\begin{theorem}[Main content theorem \ref{thm:cucb-bob} restated]\label{thm:cucb-bob-res}
	Suppose that there exist $R_1,R_2$ such that $R_1 \le r_S(\mathbf 0) \le r_S(\mathbf 1) \le R_2$ for any $S\in\mathbb S$ and $R = R_2 - R_1$. Choosing $L = \sqrt{mKT} / R$, we have the following distribution-independent regret bound for $\text{Reg}_{\alpha,\beta}$,
	\[\tilde{O}\left((m{V})^{\frac{1}{3}}(KT)^{\frac{2}{3}} + \sqrt{R}(mK)^{\frac{1}{4}}T^{\frac{3}{4}} +R\sqrt{mKT}\right).\]
	Choosing $L = K^{2/3}T^{1/3}$, we have the following distribution-dependent regret bound
	\[\tilde{O}\left(K\sqrt{\sum_{i\in [m]}\frac{TV}{\Delta^i_{\min}}} + \sum_{i\in [m]}\frac{K^{\frac{1}{3}}T^{\frac{2}{3}}}{\Delta^i_{\min}} + RK^{\frac{1}{3}}T^{\frac{2}{3}}\right).\]
\end{theorem}

\begin{proof}
    We suppose that each block has length $L$, and there are $\lceil\frac{T}{L}\rceil$ blocks in total. Then, the reward in each block is bounded by $R' = RL$, since the reward in each round is bounded by $R$. We also know that the total number of possible length of sliding window is $K' = \lceil\log_2 L\rceil$, and the time horizon for the EXP3.P algorithm is $T' = \lceil\frac{T}{L}\rceil$.
    
    From the definition of the $(\alpha,\beta)$-approximation regret, we have
    \begin{align*}
        \text{Reg}^{\cA}_{\boldsymbol\mu,\alpha,\beta} =& \alpha\cdot\beta\cdot\sum_{t=1}^T\text{opt}_{\boldsymbol\mu_t} - \E\left[\sum_{t=1}^T r_{S_t^{\cA}}(\boldsymbol\mu_t)\right] \\
        =& \underbrace{\alpha\cdot\beta\cdot\sum_{t=1}^T\text{opt}_{\boldsymbol\mu_t} - \E\left[\sum_{t=1}^T r_{S_t^{\cB}}(\boldsymbol\mu_t)\right]}_{\text{Term }\mathbb A} + \underbrace{\E\left[\sum_{t=1}^T r_{S_t^{\cB}}(\boldsymbol\mu_t)\right] - \E\left[\sum_{t=1}^T r_{S_t^{\cA}}(\boldsymbol\mu_t)\right]}_{\text{Term }\mathbb B},
    \end{align*}
    where $\cB$ is another algorithm with the same block size but with fixed window size $w = 2^k$ for some number $k$. From Proposition \ref{prop:exp3p-regret-res}, it is easy to know that for any fixed window size $w$ and the induced algorithm $\cB$, the second term (Term $\mathbb B$) is bounded by
    \[\text{Term }\mathbb B \le \tilde O(R'\sqrt{K'T'}) = \tilde O\left(RL\sqrt{\frac{T}{L}}\right) = \tilde O\left(R\sqrt{TL}\right).\]
    
    Then, the remaining part is to select a window size $w$ and bound Term $\mathbb A$. We decompose Term $\mathbb A$ into sum of regret of each block,
    \[\text{Term }\mathbb A = \alpha\cdot\beta\cdot\sum_{t=1}^T\text{opt}_{\boldsymbol\mu_t} - \E\left[\sum_{t=1}^T r_{S_t^{\cB}}(\boldsymbol\mu_t)\right] = \sum_{\ell=1}^{\lceil\frac{T}{L}\rceil}\left(\alpha\cdot\beta\cdot\sum_{s=L(\ell-1)+1}^{\min\{\ell L,T\}}\text{opt}_{\boldsymbol\mu_t} - \E\left[\sum_{s=L(\ell-1)+1}^{\min\{\ell L,T\}} r_{S_t^{\cB}}(\boldsymbol\mu_t)\right]\right).\]
    Suppose that in each block $\ell\le \lceil\frac{T}{L}\rceil$, the variation in block $\ell$ is denoted by ${V}_{\ell}$. Formally, we define
    \[{V}_{\ell} = \sum_{s=L(\ell-1)+2}^{\min\{\ell L,T\}}||\bmu_s - \bmu_{s-1}||_{\infty}.\]
    Now we bound the regret in each block. The bound is similar to the proof in Theorem \ref{thm:cucb-sw-prob-res}. Choosing $w = 2^k$ where $2^{k} \le \min\{m^{1/3}T^{2/3}K^{-1/3}{V}^{-2/3},L\} < 2^{k+1}$ and $M_i = \sqrt{mK/w}$. If we have $m^{1/3}T^{2/3}K^{-1/3}{V}^{-2/3}\le L$, then the regret in block $\ell < \frac{T}{L}$ is bounded by
    \[\tilde O\left((m{V})^{1/3}K^{2/3}T^{-1/3}\cdot L+m^{1/3}(KT)^{2/3}{V}^{-2/3}\cdot{V}_{\ell} + mK\right).\]
    The regret in last block is bounded by $L$, and Term $\mathbb A$ can be bounded by
    \[\tilde O\left((m{V})^{1/3}(KT)^{2/3}+L + mK\frac{T}{L}\right).\]
    Then we consider the case when $(mK)^{1/3}T^{2/3}{V}^{-2/3} > L$. This time, the regret in each block is bounded by
    \[\tilde O\left(\sqrt{mKL} + mK\right).\]
    Then sum the regret in each block, we bound Term $\mathbb A$ by the following
    \[\tilde O\left(\sqrt{mKL}\frac{T}{L} + L + mK\frac{T}{L}\right) = \tilde O\left(\sqrt{mK/L}\cdot T + L +mK\frac{T}{L}\right),\]
    where the last term is the regret for the last block. Sum them up, we know that Term $\mathbb A$ is bounded by
    \[\text{Term }\mathbb A \le \tilde O\left((m{V})^{1/3}(KT)^{2/3} + \sqrt{mK/L}\cdot T + L +mK\frac{T}{L}\right).\]
    Then combining Term $\mathbb B$, we have
    \[\text{Reg}^{\cA}_{\alpha,\beta} = \tilde O\left((m{V})^{1/3}(KT)^{2/3} + \sqrt{mK/L}\cdot T + L + R\sqrt{TL} +mK\frac{T}{L}\right).\]
    Choosing $L = \sqrt{mKT} / R$, the regret is bounded by
    \[\text{Reg}^{\cA}_{\alpha,\beta} = \tilde O\left((m{V})^{1/3}(KT)^{2/3} + \sqrt{R}(mK)^{1/4}T^{3/4} +R\sqrt{mKT}\right).\]
    Next, we consider the distribution dependent bound. Now, we choose $w = 2^k$ where $2^k \le\min\left\{\sqrt{\frac{T}{{V}}\cdot \sum_{i\in [m]}\frac{1}{\Delta^i_{\min}}},L\right\} < 2^{k+1}$. First we consider the case when $\sqrt{\frac{T}{{V}}\cdot \sum_{i\in [m]}\frac{1}{\Delta^i_{\min}}} \le L$. In this case, the regret in block $\ell$ (except for the last one) is bounded by
    \[\tilde O\left(\frac{L}{w}\cdot \sum_{i\in [m]}\frac{K}{\Delta^i_{\min}} + w\cdot {V}_{\ell} + mK\right).\]
    Summing up the regret in each block, we can know that Term $\mathbb A$ in this case is bounded by
    \[\tilde O\left(K\sqrt{T{V}\cdot \sum_{i\in [m]}\frac{1}{\Delta^i_{\min}}} + mKL\right).\]
    Then consider the case when $\sqrt{\frac{T}{{V}}\cdot \sum_{i\in [m]}\frac{1}{\Delta^i_{\min}}} > L$. In this case, the regret for block $\ell$ is bounded by
    \[\tilde O\left(\sum_{i\in [m]}\frac{K}{\Delta^i_{\min}} + mK\right).\]
    Summing up the regret in each block, we know that Term $\mathbb A$ is bounded by
    \[\tilde O\left(\frac{T}{L}\cdot\sum_{i\in [m]}\frac{K}{\Delta^i_{\min}} + mK\frac{T}{L}\right).\]
    Combining the regret bound in each case, we know that
    \[\text{Term }\mathbb A = \tilde O\left(K\sqrt{T{V}\cdot \sum_{i\in [m]}\frac{1}{\Delta^i_{\min}}} + \frac{T}{L}\cdot\sum_{i\in [m]}\frac{K}{\Delta^i_{\min}} + mK\frac{T}{L}\right).\]
    Take Term $\mathbb B$ into account, we have
    \[\text{Reg}^{\cA}_{\alpha,\beta} = \tilde O\left(K\sqrt{T{V}\cdot \sum_{i\in [m]}\frac{1}{\Delta^i_{\min}}} + \frac{T}{L}\cdot\sum_{i\in [m]}\frac{K}{\Delta^i_{\min}} + mK\frac{T}{L} + R\sqrt{TL}\right).\]
    Choosing $L = K^{2/3}T^{1/3}$, we can get
    \[\text{Reg}^{\cA}_{\alpha,\beta} = \tilde O\left(K\sqrt{T{V}\cdot \sum_{i\in [m]}\frac{1}{\Delta^i_{\min}}} + \sum_{i\in [m]}\frac{K^{\frac{1}{3}}T^{\frac{2}{3}}}{\Delta^i_{\min}} + RK^{\frac{1}{3}}T^{\frac{2}{3}}\right).\]
\end{proof}

\section{More Details in Section 4}\label{sec:special-detail}
\subsection{Detailed Algorithm}
In this part, we give our full algorithm pseudo-code. Please see Algorihtm \ref{alg: LCMAB} for more details.

\begin{algorithm}
   \caption{\textsc{Ada-LCMAB}}
   \label{alg: LCMAB}
\begin{algorithmic}[1]
   \STATE {\bfseries Input:} confidence $\delta$, time horizon $T$, action space $\mb{S}$
   \STATE {\bfseries Definition:} $\nu_j=\sqrt{\frac{C_0}{m2^jL}}$, where $C_0=\ln\left(\frac{8T^3|\mb{S}|^2}{\delta}\right)$, $L=\lceil 4mC_0 \rceil, \mc{B}_{(i,j)}:=[\iota_i, \iota_i+2^jL-1]$. 
   \STATE {\bfseries Initialize:} $t=1, i=1$
   \STATE $ \iota_i \leftarrow t$ \label{alg_line: init}
   \FOR{$j=0,1,2,\dots$}
   \STATE If $j=0$, set $Q_{(i,j)}$ as an arbitrary distribution over $\mb{S}$; otherwise, let $(\bm{q}_{(i,j)}^{\nu_j}, Q_{(i,j)}^{\nu_j})$ be the associated solution and distribution of equation (\ref{eq: FTRL}) with inputs $\mc{I}=\mc{B}_{(i,j-1)}$ and $\nu = \nu_j$
   \STATE $\mc{E} \leftarrow \emptyset$
   \WHILE{$t \leqslant \iota_i+2^jL-1$}
   \STATE Draw $\mathrm{REP} $ $\sim $ $\mathrm{Bernoulli}\left(\frac{1}{L}\times2^{-j/2}\times \sum_{k=0}^{j-1}2^{-k/2}\right)$ 
   \IF{$\mr{REP}=1$}
   \STATE Sample $n$ from $\{0,\dots, j-1\}$ s.t. $\Pr[n=b]\propto 2^{-b/2}$
   \STATE $\mc{E} \leftarrow \mc{E} \cup \{(n,[t,t+2^nL-1])\}$
   \ENDIF
   \STATE Let $N_t:=\{n|\exists \mc{I} \text{ such that } t \in \mc{I} \text{ and } (n, \mc{I})\in \mc{E}\}$
   \STATE If $N_t$ is empty, play $S_t \sim Q_{(i,j)}^{\nu_j}$; otherwise, sample $n \sim \text{Uniform}(N_t)$, and play $S_t \sim Q^{\nu_n}_{(i,n)}$ \label{alg: LCMAB strategy}
   \STATE Receive $\{X_i^t|i \in S_t\}$ and calculate $\hat{{\mu}}_t$ according to equation (\ref{eq: importance weight})
   \FOR{$(n,[s,s']) \in \mc{E}$}
   \IF{$s'=t$ and \textsc{EndOfReplayTest}$(i,j,n,[s,t])=Fail$}
   \STATE $t \leftarrow t+1, i\leftarrow i+1$ and return to Line \ref{alg_line: init}
   \ENDIF
   \ENDFOR
   \IF{$t=\iota_i+2^jL-1$ and \textsc{EnfOfBlockTest}$(i,j)=Fail$}
   \STATE $t\leftarrow t+1, i\leftarrow i+1$ and return to Line \ref{alg_line: init}
   \ENDIF
   \ENDWHILE
   \ENDFOR
\end{algorithmic}
\begin{algorithmic}
    \STATE {\bfseries Procedure:} \textsc{EndOfReplayTest}($i,j,n,\mc{A}$):
    \STATE Return \textit{Fail} if there exists $S \in \mb{S}$ such that any of the following inequalities holds: 
    \begin{align}
        \widehat{\mr{Reg}}_{\mc{A}}(S) - 4 \widehat{\mr{Reg}}_{\mc{B}(i,j-1)}(S) \geqslant 34 mK \nu_n \log T \label{check: replay1}\\
        \widehat{\mr{Reg}}_{\mc{B}(i,j-1)}(S) - 4\widehat{\mr{Reg}}_{\mc{A}}(S) \geqslant 34 mK \nu_n \log T \label{check: replay2}
    \end{align}
\end{algorithmic}
\begin{algorithmic}
    \STATE {\bfseries Procedure:} \textsc{EndOfBlockTest}($i,j$):
    \STATE Return \textit{Fail} if there exists $k \in \{0,1,\dots, j-1\}$ and $S \in \mb{S}$ such that any of the following inequalities holds:
    \begin{align}
        \widehat{\mr{Reg}}_{\mc{B}(i,j)}(S) - 4 \widehat{\mr{Reg}}_{\mc{B}(i,k)}(S) \geqslant 20 mK \nu_k \log T \label{check: block1}\\
        \widehat{\mr{Reg}}_{\mc{B}(i,k)}(S) - 4\widehat{\mr{Reg}}_{\mc{B}(i,j)}(S) \geqslant 20 mK \nu_k \log T \label{check: block2}
    \end{align}
\end{algorithmic}
\end{algorithm}

\subsection{Omitted Proofs in Section 4}

\begin{lemma}[Main content lemma \ref{lem: FTRL} restated]\label{lem: FTRL-res}
	For any time interval ${I}$, its empirical reward estimation $\hat{{\mu}}_{{I}}$, and exploration parameter $\nu>0$, let $\bm{q}^\nu_{{I}}$ be the solution to following optimization problem (\ref{eq: FTRL}) with constant $C=100$:
	\begin{equation}
	\label{eq: FTRL}
	\bm{q}^\nu_{{I}} = \argmax_{\bm{q}\in\mr{Conv}(\mb{S})_\nu} \inner{\bm{q}, \hat{{\bmu}}_{{I}}} + C \nu\sum_{i=1}^m \log q_i 
	\end{equation}
	Let $Q^\nu_{{I}}$ be the distribution over $\mb{N}$ such that $\mb{E}_{S\sim Q^\nu_{{I}}}[\one_S] = \bm{q}^\nu_{{I}}$, then there is 
	
	\begin{flalign}
	\sum_{S\in\mb{S}} Q^\nu_{{I}}(S) \widehat{\mr{Reg}}_{{I}}(S) \leqslant Cm\nu \label{ineq: small regret}\\ 
	\forall S \in \mb{S}, ~ \mr{Var}(Q^\nu_{{I}}, S) \leqslant m+\frac{\widehat{\mr{Reg}}_{{I}}(S)}{C\nu} \label{ineq: small variance}
	\end{flalign}
\end{lemma}

\begin{proof}
Define loss function $F_{\mc{I}}(Q):=\sum_{S\in \mb{S}}Q(S)\widehat{\mr{Reg}}_{\mc{I}}(S)+C\nu \sum_{i=1}^m \ln(1/q_i)$ with decision domain $\Delta(\mb{S})_\nu := \{Q \in \mb{R}^{|\mb{S}|}_+ | \sum_{S\in\mb{S}}Q(S) = 1, \forall i \in [m], q_i \geqslant \nu\}$ (recall $\bm{q}$ is the expectation vector of $Q$). Because the decision domain $\Delta(\mb{S})_\nu$ is compact and loss function $F_{\mc{I}}(Q)$ is strictly convex in $\Delta(\mb{S})_\nu$, there exists a unique minimizer. What's more, it is not difficult to see $Q_{\mc{I}}^\nu$ induced by the solution to equation (\ref{eq: FTRL}) is exactly the minimizer of loss function $F_{\mc{I}}(Q)$. Now we prove the lemma.

Define $\Delta(\mb{S})'_{\nu} := \{Q \in \mb{R}^{|\mb{S}|}_+ | \sum_{S\in\mb{S}}Q(S) \leqslant 1,  \forall i \in [m], q_i \geqslant \nu\}$. We claim there is $\min_{Q\in\Delta(\mb{S})} F_{\mc{I}}(Q) = \min_{Q\in\Delta(\mb{S})'} F_{\mc{I}}(Q)$, otherwise we can increase the weight of $\hat{S}_{\mc{I}}$ in $\Delta(\mb{S})'_\nu$ until it reaches the boundary, which always decreases the loss value.

Since $\nabla F_{\mc{I}}(Q)|_{Q(S)} = \widehat{\mr{Reg}}_{\mc{I}}(S) - Cv \sum_{i\in S}1/q_i$, according to KKT conditions, we have 
\begin{align}
\label{eq:KKT}
    \widehat{\mr{Reg}}_{\mc{I}}(S) - C\nu \sum_{i\in S} \frac{1}{\bm{q}^\nu_{\mc{I},i}} - \lambda_S - \sum_{i\in S}\lambda_{i}+ \lambda = 0
\end{align}
for some Lagrangian multipliers $\lambda_S \geqslant 0, \lambda_{i} \geqslant 0, \lambda \geqslant 0$. Multiplying both sides by $Q_{\mc{I}}^\nu(S)$ and summing over $S \in \mb{S}$ give 
\begin{align*}
    \sum_{S\in\mb{S}} Q_{\mc{I}}^\nu(S) \widehat{\mr{Reg}}_{\mc{I}}(S) &= C\nu \sum_{S\in\mb{S}} Q_{\mc{I}}^\nu(S) \sum_{i\in S} \frac{1}{\bm{q}^\nu_{\mc{I},i}} + \sum_{S\in\mb{S}} Q_{\mc{I}}^\nu(S) \lambda_S + \sum_{S\in\mb{S}} \sum_{i\in S} Q_{\mc{I}}^\nu(S) \lambda_{i} - \lambda \\
    &= C\nu \sum_{S\in\mb{S}} Q_{\mc{I}}^\nu(S) \sum_{i\in S} \frac{1}{\bm{q}^\nu_{\mc{I},i}} - \lambda \\
    &=Cm\nu - \lambda \\
    &\leqslant Cm\nu
\end{align*}
where the second equality is because of complementary slackness. Now we have proved the inequality (\ref{ineq: small regret}) stated in the theorem. What's more, as $\widehat{\mr{Reg}}_{\mc{I}}(S) \geqslant 0$ for $\forall S \in \mb{S}$, there is $\lambda \leqslant Cm\nu$.

Rearranging from equation (\ref{eq:KKT}), we know 
\begin{align*}
    \sum_{i\in S} \frac{1}{\bm{q}^\nu_{\mc{I},i}} &= \frac{1}{C\nu} \left(\widehat{\mr{Reg}}_{\mc{I}}(S)- \lambda_S - \sum_{i\in S} \lambda_i + \lambda\right)\\
    &\leqslant m+\frac{\widehat{\mr{Reg}}_{\mc{I}}(S)}{C\nu}
\end{align*}
which finishes the proof of inequality (\ref{ineq: small variance}).
\end{proof}

For any interval $\mc{I}$ that lies in a block $j$ of epoch $i$ (i.e. $[\iota_i+2^{j-1}L,\iota_i+2^jL-1]$), define $\varepsilon_{\mc{I}}:=\max_{S\in\mb{S}}\mr{Reg}_{\mc{I}}(S) - 8\widehat{\mr{Reg}}_{\mc{B}_{(i,j-1)}}(S), \alpha_{\mc{I}}=\sqrt{\frac{2mC_0}{|\mc{I}|}}\log_2 T$, where $\mr{Reg}_{\mc{I}}(S):= \sum_{t\in\mc{I}} \mr{opt}_{{\mu}_t} - r_{S_t}({\mu}_t)$. In Lemma \ref{lem: interval regret} and Lemma \ref{lem: block regret}, since we consider the regret in epoch $i$, we use $\mc{B}_j$ to represent $\mc{B}_{(i,j)}$ for simplicity.

\begin{lemma}
\label{lem: interval regret}
With probability $1-\delta$, \textsc{Ada-LCMAB} guarantees for any block $j$ and any interval $\mc{I}$ lies in block $j$,
\begin{align*}
    \sum_{t\in\mc{I}}\mr{opt}_{{\mu}_t}-r_{S_t}({\mu}_t) = \tilde{{O}}\left(|\mc{I}|mK\nu_n+|\mc{I}|(K\alpha_{\mc{I}}+K\Delta_{\mc{I}}+\epsilon_{\mc{I}}\mb{I}_{\varepsilon_{\mc{I}}>D_3K\alpha_{\mc{I}}})\right)
\end{align*}
where $D_3= 170.$
\end{lemma}
\begin{proof}
First, according to Azuma's inequality and a union bound over all $T^2$ intervals, with probability $1-\delta$, for any interval $\mc{I}$, there is 
\begin{align}
    \sum_{t\in\mc{I}}\mr{opt}_{{\mu}_t}-r_{S_t}({\mu}_t) \leqslant \sum_{t\in\mc{I}}\mb{E}_t[\mr{opt}_{{\mu}_t}-r_{S_t}({\mu}_t)] + {O}\left(K\sqrt{|\mc{I}|\log(T^2/\delta)}\right)
\end{align}

Now we bound the conditional expectation in above inequality.

Note
\begin{align}
    \mb{E}_t[\mr{opt}_{{\mu}_t}-r_{S_t}({\mu}_t)] & = 
    \begin{cases}
    \sum_{S\in\mb{S}} Q_j^{\nu_j}(S)(\mr{opt}_{{\mu}_t}-r_{S}({\mu}_t)) & \text{if } N_t=\emptyset \\
    \sum_{S\in\mb{S}}\sum_{n\in N_t} \frac{Q_n^{\nu_n}(S)}{|N_t|}(\mr{opt}_{{\mu}_t}-r_{S}({\mu}_t)) & \text{if } N_t\neq \emptyset
    \end{cases} \\
    & = \begin{cases}
    \sum_{S\in\mb{S}} Q_j^{\nu_j}(S)\mr{Reg}_t(S) & \text{if } N_t=\emptyset \\
    \sum_{S\in\mb{S}}\sum_{n\in N_t} \frac{Q_n^{\nu_n}(S)}{|N_t|}\mr{Reg}_t(S) & \text{if } N_t\neq \emptyset
    \end{cases} 
\end{align}
Now, for any $t \in \mc{I}$ and $n\in [j]$, there is 
\begin{align*}
    \sum_{S\in\mb{S}}Q_n^{\nu_n}(S) \mr{Reg}_t(S) & \leqslant \sum_{S\in\mb{S}}Q_n^{\nu_n}(S) \mr{Reg}_{\mc{I}}(S) + {O}(K\Delta_\mc{I}) \quad \quad \text{(nearly the same as Lemma 8 in \cite{chen2019new})} \\
    & = 8 \sum_{S\in\mb{S}}Q_n^{\nu_n}(S) \widehat{\mr{Reg}}_{\mc{B}_{j-1}}(S) + {O}(K\Delta_\mc{I}) + \varepsilon_{\mc{I}} \\
    & \leqslant 8 \sum_{S\in\mb{S}}Q_n^{\nu_n}(S) \left(4\widehat{\mr{Reg}}_{\mc{B}_{n-1}}(S) +20mK\nu_{n-1} \log T\right) + {O}(K\Delta_\mc{I}) + \varepsilon_{\mc{I}}\\
    & \quad \quad \text{(condition (\ref{check: block1}) doesn't hold)}\\
    & \leqslant \tilde{{O}}(mK\nu_n+K\Delta_{\mc{I}})+\epsilon_{\mc{I}} \\
    & \leqslant \tilde{{O}}(mK\nu_n+K\Delta_{\mc{I}}+K\alpha_{\mc{I}})+\varepsilon_{\mc{I}}\mb{I}_{\varepsilon_{\mc{I}}>D_3K\alpha_{\mc{I}}} 
\end{align*}
Combining all above inequalities and using the fact $\sqrt{|\mc{I}|\log(T^2/\delta)}\leqslant {O}(|\mc{I}|\alpha_{\mc{I}})$ finish the proof.
\end{proof}

Next, we bound the dynamic regret in block $j$ within epoch $i$, that is $\mc{J}:= [\iota_i, \iota_{i+1}-1]\cap [\iota_i+2^{j-1}L, \iota_i+2^jL-1]$.

\begin{lemma}
\label{lem: block regret}
With probability $1-\delta$, Algorithm \ref{alg: LCMAB} has the following regret for any block $\mc{J}$:
\begin{align*}
    \sum_{t\in\mc{J}}(\mr{opt}_{{\mu}_t}-r_{S_t}({\mu}_t)) = \tilde{{O}}\left(\min \left\{\sqrt{mC_0\mc{S}_{\mc{J}}|\mc{J}|}, \sqrt{mC_0|\mc{J}|}+C_0^{\frac{1}{3}}m^{\frac{4}{3}}\Delta_{\mc{J}}^{\frac{1}{3}}|\mc{J}|^{\frac{2}{3}}\right\}\right)
\end{align*}
\end{lemma}

To prove this lemma, we first partition the block into several intervals with some desired properties. As the greedy algorithm in \cite{chen2019new} used to partition the block $\mc{J}$ is only based on the total variation of underlying distribution, we can directly use the same greedy algorithm in non-stationary CMAB and have the same result:

\begin{lemma}[Lemma 5 in \cite{chen2019new}]
\label{lem: partition}
There exists a partition $\mc{I}_1 \cup \mc{I}_2 \cup \cdots \cup \mc{I}_{\Gamma}$ of block $\mc{J}$ such tht $\Delta_{\mc{I}_k} \leqslant \alpha_{\mc{I}_k}, \forall k \in [\Gamma]$, and $\Gamma = {O}(\min \{\mc{S}_{\mc{J}}, (mC_0)^{-\frac{1}{3}}\Delta_{\mc{J}}^{\frac{2}{3}}|\mc{J}|^{\frac{1}{3}}+1\})$
\end{lemma}

Next, we give some basic concentration results for Linear CMAB. Define $U_t(S):=\mb{E}_t[(r_{S}(\hat{{\mu}}_t) - r_S({\mu}_t))^2]$.

\begin{lemma}
\label{lem: conditional variance}
For any $S\in\mb{S}$ and any time $t$ in epoch $i$ and block $j$, there is 
\begin{align*}
    U_t(S) \leqslant 
    \begin{cases}
    K\mr{Var}(Q_{(i,n)}^{\nu_n},S)\log T ~ (\forall n \in [N_t]) & \text{if } N_t \neq \emptyset \\
    K\mr{Var}(Q_{(i,j)}^{\nu_j},S) & \text{if } N_t = \emptyset
    \end{cases}
\end{align*}
\end{lemma}
\begin{proof}
If $N_t \neq \emptyset$, then $U_t(S) \leqslant \mb{E}_t[r^2_{S}(\hat{{\mu}}_t)]=\mb{E}_t[(\hat{{\mu}}_t^\top \one_S)^2]\leqslant K \sum_{k\in S} \mb{E}_t[\hat{{\mu}}_{t,k}^2]\leqslant K \sum_{k\in S} \frac{1}{q_{t,k}}$, where $\bm{q}_t$ is the expectation of distribution $Q_t$ played at round $t$. According to our Algorithm \ref{alg: LCMAB}, we know $Q_t = \frac{1}{|N_t|}\sum_{n\in N_t} Q_{(i,n)}^{\nu_n}$ when $N_t \neq \emptyset$. Thus, $\bm{q}_t = \frac{1}{|N_t|}\sum_{n\in N_t} \bm{q}_{(i,n)}^{\nu_n}$ where $\bm{q}_{(i,n)}^{\nu_n}$ is the expectation of distribution $Q_{(i,n)}^{\nu_n}$, and $q_{t,k} \geqslant q_{(i,n), k}^{\nu_n} / |N_t|$. What's more, as $|N_t|\leqslant log T$, we then finish the proof when $N_t \neq \emptyset$. If $N_t$ is empty, the proof is exactly the same.
\end{proof}

\begin{lemma}
\label{ineq: concentration of regret}
With probability at least $1-\delta/4$, for any $S \in \mb{S}$, we have 
\begin{align*}
    |r_S(\hat{{\mu}}_{\mc{B}_{(i,j)}}) - r_S({\mu}_{\mc{B}_{(i,j)}})| \leqslant \frac{\lambda}{|\mc{B}_{(i,j)}|} \sum_{t\in\mc{B}_{(i,j)}} U_t(S) + \frac{C_0}{\lambda |\mc{B}_{(i,j)}|} \quad (\forall \lambda \in (0, \frac{\nu_j}{K}])
\end{align*}
and for any interval $\mc{A}$ covered by some replay phase of index $n$,
\begin{align*}
    |r_S(\hat{{\mu}}_{\mc{A}}) - r_S({\mu}_{\mc{A}})| \leqslant \frac{\lambda}{|{\mc{A}}|} \sum_{t\in{\mc{A}}} U_t(S) + \frac{C_0}{\lambda |{\mc{A}}|} \quad (\forall \lambda \in (0, \frac{\nu_n}{K}])
\end{align*}
\end{lemma}
\begin{proof}
Using Freedman's inequality with respect to each term in the summation just like Lemma 14 in \cite{chen2019new}.
\end{proof}
Define \textsc{Event}$_1$ as the event that bounds in Lemma \ref{ineq: concentration of regret} holds, then \textsc{Event}$_1$ holds with probability at least $1-\delta/4$.

\begin{lemma}
\label{lem: concentration1}
Assume \textsc{Event}$_1$ holds, and there is no restart triggered in $\mc{B}_j$, then the following hold for any $S \in \mb{S}$:
\begin{align*}
    \mr{Reg}_{\mc{B}_j}(S) \leqslant 2 \widehat{\mr{Reg}}_{\mc{B}_j}(S) + 10 mK \nu_j \\
    \widehat{\mr{Reg}}_{\mc{B}_j}(S) \leqslant 2 \mr{Reg}_{\mc{B}_j}(S) + 10 mK \nu_j
\end{align*}
\end{lemma}
\begin{proof}
We prove this lemma by induction. When $j=0$, it's not hard to see $\mr{Reg}_{\mc{B}_0}(S) \leqslant K \leqslant 10mK \nu_0$, 
\begin{align*}
    \widehat{\mr{Reg}}_{\mc{B}_0}(S) - \mr{Reg}_{\mc{B}_0}(S) & = r_{\hat{S}_{\mc{B}_0}}(\hat{{\mu}}_{\mc{B}_0}) - r_{S}(\hat{{\mu}}_{\mc{B}_0})  - r_{S_{\mc{B}_0}}({\mu}_{\mc{B}_0}) + r_{S}({\mu}_{\mc{B}_0}) \\
    & \leqslant r_{\hat{S}_{\mc{B}_0}}(\hat{{\mu}}_{\mc{B}_0}) - r_{S}(\hat{{\mu}}_{\mc{B}_0})  - r_{\hat{S}_{\mc{B}_0}}({\mu}_{\mc{B}_0}) + r_{S}({\mu}_{\mc{B}_0}) \quad (\text{by the optimality of $S_{\mc{B}_0}$}) \\
    & \leqslant 2\left(\frac{\nu_0}{KL} \sum_{t\in{\mc{B}_0}} U_t(S) + \frac{KC_0}{\nu_0 L}\right) \quad (\text{by the definition of \textsc{Event}$_1$ with $\lambda = \nu_0/K$ } ) \\
    & \leqslant 2(K+K/2) \\
    & \leqslant 4K
\end{align*}
which implies $\widehat{\mr{Reg}}_{\mc{B}_0}(S) \leqslant 5K \leqslant 10mK \nu_0$.

Now, assume the inequalities hold for $\{0,\dots, j-1\}$, then for any $t \in \mc{B}_j$ and any $n \in [1,j]$, there is 
\begin{align*}
    \mr{Var}(Q_n^{\nu_n}, S) & \leqslant m+\frac{\widehat{\mr{Reg}}_{\mc{B}_{n-1}}(S)}{C\nu_n} \\
    & \leqslant m+\frac{2\mr{Reg}_{\mc{B}_{n-1}}(S)+ 10mK\nu_{n-1}}{C\nu_n} \\
    & \leqslant \frac{\mr{Reg}_{\mc{B}_{n-1}}(S)}{3\nu_n} + mK \\
    & \leqslant \frac{\mr{Reg}_{\mc{B}_{n-1}}(S)}{3\nu_j} + mK   
\end{align*}
Combining Lemma \ref{lem: conditional variance} above and Lemma 19 in \cite{chen2019new} gives the result in this theorem.
\end{proof}
\begin{lemma}
\label{lem: concentration2}
Assume \textsc{Event}$_1$ holds. Let $\mc{A}$ be a complete replay phase of index $n$, if for any $S \in \mc{S}$, equation (\ref{check: replay2}) in EndofReplayTest doesn't hold, then the following hold for all $S \in \mb{S}$:
\begin{align*}
    \mr{Reg}_{\mc{A}}(S) \leqslant 2 \widehat{\mr{Reg}}_{\mc{A}}(S) + C_3 mK \nu_n  \\
    \widehat{\mr{Reg}}_{\mc{A}}(S) \leqslant 2 \mr{Reg}_{\mc{A}}(S) + C_3 mK \nu_n
\end{align*}
where $C_3=15$
\end{lemma}
\begin{proof}
According to Lemma \ref{lem: interval regret} and Lemma \ref{lem: conditional variance}, we have 
\begin{align*}
    \mr{Var}(Q_n^{\nu_n},S) & \leqslant m + \frac{\widehat{\mr{Reg}}_{\mc{B}_{n-1}}(S)}{C\nu_n} \\
    & \leqslant m + \frac{4\widehat{\mr{Reg}}_{\mc{B}_{j-1}}(S)+20mK\nu_n\log T}{C\nu_n} \\
    & \leqslant \frac{30\log T}{C}mK + \frac{16\widehat{\mr{Reg}}_{\mc{A}}(S)+136mK\nu_n\log T}{C\nu_n} \quad \text{( because of EndOfReplayTest)} \\
    & \leqslant \frac{\widehat{\mr{Reg}}_{\mc{A}}(S)}{3\nu_n} + \frac{166\log T}{C} mK
\end{align*}
Combining Lemma \ref{lem: conditional variance} and Lemma 19 in \cite{chen2019new} proves the result.
\end{proof}

\begin{lemma}
\label{lem: concentration3}
Assume \textsc{Event}$_1$ holds. Let $\mc{A} = [s,e]$ be a complete replay phase of index $n$, then the following hold for all $S \in \mb{S}$:
\begin{align*}
    \mr{Reg}_{\mc{A}}(S) \leqslant 2 \widehat{\mr{Reg}}_{\mc{A}}(S) + 4 mK \nu_n + \bar{{V}}_{[\iota_i,e]}  \\
    \widehat{\mr{Reg}}_{\mc{A}}(S) \leqslant 2 \mr{Reg}_{\mc{A}}(S) + 4 mK \nu_n + \bar{{V}}_{[\iota_i,e]}
\end{align*}
\end{lemma}
\begin{proof}
    For any $t \in \mc{A}$, there is 
    \begin{align*}
    \mr{Var}(Q_n^{\nu_n},S) & \leqslant m + \frac{\widehat{\mr{Reg}}_{\mc{B}_{n-1}}(S)}{C\nu_n} \\
    & \leqslant m + \frac{2\mr{Reg}_{\mc{B}_{n-1}}(S)+10mK\nu_n}{C\nu_n} \quad \text{( because of Lemma \ref{lem: concentration1})}\\
    & \leqslant \frac{1}{2}mK + \frac{2\mr{Reg}_{\mc{A}}(S)+2m\bar{{V}}_{[\iota_i,e]}}{C\nu_n} \quad \text{( because of Lemma 8 in \cite{chen2019new})} \\
    & \leqslant \frac{\mr{Reg}_{\mc{A}}(S)}{3\nu_n} + \frac{1}{2} mK + \frac{2m\bar{{V}}_{[\iota_i,e]}}{C\nu_n}
\end{align*}
Combining Lemma \ref{lem: conditional variance} above and Lemma 19 in \cite{chen2019new} proves the result.
\end{proof}

\begin{lemma}
    \label{lem: critical lemma}
    Assume \textsc{Event}$_1$ holds. Let $\mc{I}=[s,e]$ be an interval in the fictitious block $\mc{J}^{'}$ with index $j$, and such that $\bar{V}_{\mc{I}}\leqslant \alpha_{\mc{I}}, \epsilon_{\mc{I}}>D_3K\alpha_{\mc{I}}$, then
    \begin{itemize}
        \item[(1)] there exist an index $n_{\mc{I}} \in \{0,1,\dots, j-1\}$ such that $D_3mK\nu_{n+1}\log T \leqslant \epsilon_{\mc{I}} \leqslant D_3 mK \nu_n \log T$;
        \item[(2)] $|\mc{I}| \geqslant 2^{n_{\mc{I}}}L$;
        \item[(3)] if the algorithm starts a replay phase $\mc{A}$ with index $n_{\mc{I}}$ within the range of $[s, e-2^{n_{\mc{I}}}L]$, then the algorithm restarts when the replay phase finishes.
    \end{itemize}
\end{lemma}
\begin{proof}
    For (1), on one hand $\epsilon_{\mc{I}} \leqslant K \leqslant D_3mK\nu_0$; on the other hand, $\epsilon_{\mc{I}}>D_3K\alpha_{\mc{I}} \geqslant D_3 mK\nu_j \log T$ because of the definition of $\alpha_{\mc{I}}, \nu_j$ and $|\mc{I}| \leqslant |\mc{J}^{'}| \leqslant 2^{j-1}L$. Therefore, there must exist an index $n_{\mc{I}}$ such that the condition holds.

    For (2), since $D_3K\alpha_{\mc{I}} \leqslant D_3mK\nu_{n_{\mc{I}}}\log T$, we have $|\mc{I}| > 2^{n_{\mc{I}}}L$.

    For (3), we show that the $\textsc{EndOfReplayTest}$ fails when the replay phase finishes. Suppose for $\forall S \in \mb{S}$, Eq.(\ref{check: replay2}) doesn't hold, then according to Lemma \ref{lem: concentration2}, we know $\mr{Reg}_{\mc{A}}(S) \leqslant 2 \widehat{\mr{Reg}}_{\mc{A}}(S) + C_3 mK \nu_{n_{\mc{I}}}$. Besides, we know there exists $S'$ such that 
    \begin{align*}
        \mr{Reg}_{\mc{A}}(S') & \geqslant \mr{Reg}_{\mc{I}}(S')-2K\bar{V}_{\mc{I}} \quad (\text{because of Lemma 8 in \cite{chen2019new}}) \\
        & \geqslant 8 \widehat{\mr{Reg}}_{\mc{B}_{j-1}}(S') + \epsilon_{\mc{I}} -2K\bar{V}_{\mc{I}} \quad (\text{because of the definition of $\epsilon_{\mc{I}}$}) \\
        & \geqslant 8 \widehat{\mr{Reg}}_{\mc{B}_{j-1}}(S') + (D_3/2-2)mK\nu_{n_{\mc{I}}} \log T
    \end{align*}
    Combining above two inequalities, we have 
    \begin{align*}
        \widehat{\mr{Reg}}_{\mc{A}}(S') & > 4\widehat{\mr{Reg}}_{\mc{B}_{j-1}}(S') + \frac{0.5D_3-2-C_3}{2}mK\nu_{n_{\mc{I}}}\log T \\
        & = 4\widehat{\mr{Reg}}_{\mc{B}_{j-1}}(S') + 34mK\nu_{n_{\mc{I}}}\log T 
    \end{align*}
    which is the Eq.(\ref{check: replay1}) in \textsc{EndOfReplayTest}, thus the algorithm will restart.
\end{proof}

\begin{proof}[Proof of Lemma \ref{lem: block regret}]
    Consider the fictitious partition constructed in Lemma \ref{lem: partition}, for the first $\Gamma-1$ intervals, using Lemma \ref{lem: interval regret} with respect to each interval as there is no restart. For the last interval $\Gamma$, we also use Lemma \ref{lem: interval regret} but with the fictitious planned interval in the same way as in paper \cite{chen2019new}.

    Thus, for block $j$ (i.e. $[\iota_i, \iota_{i+1}-1]\cup[\iota_i+2^{j-1}L-1, \iota_i+2^jL-1]$), there is 
    \begin{align*}
          &\sum_{t\in\mc{J}} \mr{opt}_{{\mu}_t} - r_{S_t}({\mu}_t)\\
          \leqslant & \underbrace{\sum_{k=1}^\Gamma \sum_{t\in \mc{I}_k} \sum_{n \in N_t\cup \{j\}} mK\nu_n}_{\text{Term1}} + \underbrace{\sum_{k=1}^{\Gamma-1}K|\mc{I}_k|\alpha_{\mc{I}_k} + K|\mc{I}_{\Gamma}|\alpha_{\mc{I}^{'}_{\Gamma}}}_{\text{Term2}} + \underbrace{\sum_{k=1}^{\Gamma-1}|\mc{I}_k|\varepsilon_{\mc{I}_k}{I}_{\varepsilon_{\mc{I}_k}>D_3K\alpha_{\mc{I}_k}} + |\mc{I}_\Gamma|\varepsilon_{\mc{I}^{'}_\Gamma}{I}_{\varepsilon_{\mc{I}^{'}_\Gamma}>D_3K\alpha_{\mc{I}^{'}_\Gamma}}}_{\text{Term3}}
      \end{align*}  

    Using exactly the same technique as \cite{chen2019new} and Lemma \ref{lem: critical lemma} above, one can prove 
    \begin{align*}
         \text{Term1} & \leqslant {O}(\log(1/\delta)\sqrt{C_0mK2^jL}) \\
         \text{Term2} & \leqslant {O}(\log T \sqrt{C_0mK\Gamma |\mc{J}|}) \\
         \text{Term3} & \leqslant {O}(\log(1/\delta)\log T \sqrt{C_0mK\Gamma 2^jL}) 
    \end{align*}
    Combining all above inequalities and Lemma \ref{lem: partition} finishes the proof.
\end{proof}


\begin{theorem}[Theorem \ref{thr: LCMAB} restated]
	Algorithm \ref{alg: LCMAB} guarantees $\text{Reg}^{\cA}_{1,1}$ is upper bounded by
	\begin{equation*}
	\tilde{O}\left( \min \left\{\sqrt{mK^2NT}, \sqrt{mK^2T}+K(m\bar{V})^{\frac{1}{3}}T^{\frac{2}{3}}\right\} \right).
	\end{equation*}
\end{theorem}

\begin{proof}
    First, we bound the regret in an epoch $i$ (i.e. $\mc{H}_i=[\iota_i, \iota_{i+1}-1]$). For block $j$ in epoch $i$, we denote it as $\mc{J}_{ij}=[\iota_i+2^{j-1}L, \iota_i+2^jL-1]\cap \mc{H}_i$. As the last index of $j$ is at most $j^*=\lceil \log (|\mc{H}_i/L|) \rceil$, we have 
    \begin{align*}
        \mb{E} \left[\sum_{t\in \mc{H}_i} \mr{opt}_{{\mu}_t} - r_{S_t}({\mu}_t) \right] & \leqslant \tilde{O}\left(L+\sum_{j=1}^{j^*}\sqrt{C_0mK^2\mc{S}_{\mc{J}_{ij}}2^jL}\right) \\
        & = \tilde{O}\left( \sqrt{C_0mK^2\mc{S}_{\mc{H}_{i}}|\mc{H}_i|}\right)
    \end{align*}
    Similarily, using H\"{o}lder inequality, we have 
    \begin{align*}
        \mb{E} \left[\sum_{t\in \mc{H}_i} \mr{opt}_{{\mu}_t} - r_{S_t}({\mu}_t) \right] & \leqslant \tilde{O}\left( \sqrt{C_0mK^2|\mc{H}_i|}+KC_0^{\frac{1}{3}}m^{\frac{1}{3}}\bar{{V}}^{\frac{1}{3}}_{\mc{H}_i}|\mc{H}_i|^{\frac{2}{3}}\right)
    \end{align*}
    According to Lemma \ref{lem: number of restart} below, we know there is at most $E:=\min\{\mc{S}, (C_0m)^{-\frac{1}{3}}\bar{{V}}^{\frac{2}{3}}T^{\frac{1}{3}}+1\}$ number of epochs with high probability, thus summing up the regret bound over all epochs, we have 
    \begin{align*}
        \sum_{t=1}^T \mb{E}\left[\mr{opt}_{{\mu}_t} - r_{S_t}({\mu}_t)\right] &\leqslant \tilde{{O}}\left(\sum_{t=1}^E \sqrt{C_0mK^2 \mc{S}_{\mc{H}_i}|\mc{H}_i|}\right) \\
        & \leqslant  \tilde{{O}}\left( \sqrt{C_0mK^2\mc{S}T}\right)
    \end{align*}
    and
    \begin{align*}
        \sum_{t=1}^T \mb{E}\left[\mr{opt}_{{\mu}_t} - r_{S_t}({\mu}_t)\right] &\leqslant \tilde{{O}}\left(\sum_{t=1}^E \left(\sqrt{C_0mK^2|\mc{H}_i|}+KC_0^{\frac{1}{3}}m^{\frac{1}{3}}\bar{{V}}^{\frac{1}{3}}_{\mc{H}_i}|\mc{H}_i|^{\frac{2}{3}} \right)\right) \\
        &\leqslant \left(\sqrt{C_0mK^2T}+KC_0^{\frac{1}{3}}m^{\frac{1}{3}}\bar{{V}}^{\frac{1}{3}}T^{\frac{2}{3}}\right)
    \end{align*}
\end{proof}
\begin{lemma}
    \label{lem: number of restart}
    Denote the number of restart by $E$. With probability $1-\delta$, we have $E \leqslant \min\{\mc{S}, (C_0m)^{-\frac{1}{3}}\bar{{V}}^{\frac{2}{3}}T^{\frac{1}{3}}+1\}$.
\end{lemma}
\begin{proof}
    First, we prove that if for all $t$ in epoch $i$ with $\bar{{V}}_{[\iota_i, t]}\leqslant \sqrt{\frac{mC_0}{t-\iota_i+1}}$, restart will not be triggered at time $t$.

    For \textsc{EndOfBlockTest}, suppose $t=\iota_i+2^jL-1$ for some $j$, then for any $S\in\mc{S}, k\in [0, j-1]$, we have
    \begin{align*}
        \widehat{\mr{Reg}}_{\mc{B}_j} &\leqslant 2\mr{Reg}_{\mc{B}_j}(S) + 10 mK \nu_j \quad (\text{because of Lemma \ref{lem: concentration1}}) \\
        & \leqslant 2\mr{Reg}_{\mc{B}_k}(S) + 10 mK \nu_j + 4m\bar{{V}}_{[\iota_i,t]} \quad (\text{because of Lemma 8 in \cite{chen2019new} }) \\
        & \leqslant 4\widehat{\mr{Reg}}_{\mc{B}_k}(S) + 34 mK \nu_j \quad (\text{because of above condition and definition of $\nu_j$})
    \end{align*}
    Similarly, there is $\widehat{\mr{Reg}}_{\mc{B}_k} \leqslant 4\widehat{\mr{Reg}}_{\mc{B}_j}+34 mK \nu_j$. Thus, \textsc{EndOfBlockTest} will not return Fail.

    For \textsc{EndOfReplayTest}, suppose $\mc{A}\subset [\iota_i,t]$ be a complete replay phase of index $n$, and $\bar{{V}}_{[\iota_i, t]}\leqslant \sqrt{\frac{mC_0}{|\mc{A}|}}$, we have
    \begin{align*}
        \widehat{\mr{Reg}}_{\mc{A}} &\leqslant 2\mr{Reg}_{\mc{A}}(S) + 4 mK \nu_n + m\bar{{V}}_{[\iota_i, t]}\quad (\text{because of Lemma \ref{lem: concentration3}}) \\
        & \leqslant 2\mr{Reg}_{\mc{B}_{j-1}}(S) + 4 mK \nu_n + 5m\bar{{V}}_{[\iota_i,t]} \quad (\text{because of Lemma 8 in \cite{chen2019new} }) \\
        & \leqslant 4\widehat{\mr{Reg}}_{\mc{B}_k}(S) + 20 mK \nu_n \quad (\text{because of above condition and definition of $\nu_j$})
    \end{align*}
    Similarly, there is $\widehat{\mr{Reg}}_{\mc{B}_{j-1}} \leqslant 4\widehat{\mr{Reg}}_{\mc{B}_j}+20 mK \nu_n$. Thus, \textsc{EndOfBlockTest} will not return Fail.

    With above result, now we prove the theorem. If there is no distribution change which implies $\bar{{V}}_{[\iota_i,t]}=0$ then the algorithm will not restart. Therefore we have $E \leqslant \mc{S}$.

    Denote the length of each epoch as $T_1, \dots, T_E$, according to above result, we know there must be $\bar{{V}}_{\mc{H}_i} > \sqrt{\frac{mC_0}{T_i}}$. By H\"{o}lder's inequality, we have 
    \begin{align*}
        E-1 & \leqslant \sum_{i=1}^{E-1} T_i^{\frac{1}{3}} T_i^{-\frac{1}{3}} \\
        & \leqslant \left(\sum_{i=1}^{E-1} T_i\right)^{\frac{1}{3}} \left(\sum_{i=1}^{E-1} T_i^{-\frac{1}{2}}\right)^{\frac{2}{3}} \\
        & \leqslant T^{\frac{1}{3}}\left(\frac{\bar{{V}}}{\sqrt{mC_0}}\right)^{\frac{2}{3}} \\
        & \leqslant (mC_0)^{-\frac{1}{3}} \bar{{V}}^{\frac{2}{3}}T^{\frac{1}{3}}
    \end{align*}
\end{proof}

\subsection{Non-stationary Linear CMAB in General Case}
In section \ref{sec:special}, we need to solve an FTRL optimization probelm in Algorithm \ref{alg: LCMAB} and find a distribution $Q$ over the decision space $\mb{S}$ such that its expectation is the solution to FTRL, which can only be implemented efficiently when $\mr{Conv}(\mb{S})_{\nu}$ is described by a polynomial number of constraints \cite{zimmert2019beating,combes2015combinatorial,sherali1987constructive}. In general, the problems with polynomial number of constraints for $\mr{Conv}(\mb{S})_{\nu}$ is a subset of all the problem with linear reward function and exact offline oracle, but there are also many of them whose convex hull can be represented by polynomial number of constraints. For example, for the TOP K arm problem, the convex hull of the feasible actions can be represented by polynomial number of constraints. Another non-trivial example is the bipartite matching problem. The convex hull of all the matchings in a bipartite graph can also be represented by polynomial number of constraints. This is due to the fact that, by applying the convex relaxation of the bipartite matching problem, the constraint matrix of the corresponding linear programming is a Totally Unimodular Matrix (TUM), and the resulting polytope of the linear programming is integral, i.e. all the vertices have integer coordinates. In this way, each vertex is a feasible matching, and the polytope is the convex hull.

To make it more general and get rid of the constraint about polynomial description of $\mr{Conv}(\mb{S})_{\nu}$, instead of solving FTRL and then calculating corresponding distribution $Q$, what we need to do is to find a distribution $Q$ such that it satisfies inequalities (\ref{ineq: small regret}) and (\ref{ineq: small variance}) given in Lemma \ref{lem: FTRL-res}. In fact, we can achieve this goal using similar methods as in \citet{agarwal2014taming,chen2019new} to find a sparse distribution over $\mb{S}$ efficiently through our offline exact oracle or equivalently an ERM oracle \footnote{We also need to add a small exploration probability over $m$ super arms where $i$-th super arm contains base arm $i$ in Step 15 of Algorithm \ref{alg: LCMAB} just like \citet{chen2019new}.}.

\end{document}